%% file: main-9565-Lu.tex
\documentclass[11pt,a4paper]{article}
\usepackage{times,latexsym}
\usepackage{url}
\usepackage[T1]{fontenc}

\usepackage[acceptedWithA]{tacl2021v1}
\input{math_commands.tex}

\usepackage{xspace,mfirstuc,tabulary}

\newif\iftaclinstructions
\taclinstructionsfalse 
\iftaclinstructions
\renewcommand{\confidential}{}
\renewcommand{\anonsubtext}{(No author info supplied here, for consistency with
TACL-submission anonymization requirements)}
\newcommand{\instr}
\fi

\iftaclpubformat 

\else

\fi

\usepackage{hyperref}
\usepackage{url}
\usepackage{amsmath,amssymb}
\usepackage{enumitem}
\usepackage{algpseudocode}
\usepackage{algorithm}
\usepackage{amsthm}
\usepackage{mathtools}
\usepackage{subcaption}
\usepackage{booktabs}
\usepackage{multirow}
\usepackage{wrapfig}  
\usepackage{adjustbox,array,xcolor}

\newtheorem{theorem}{Theorem}[section]
\newtheorem{assumption}{Assumption}[section]
\newtheorem{lemma}[theorem]{Lemma}
\theoremstyle{definition}
\newtheorem{definition}{Definition}[section]
\theoremstyle{remark}
\newtheorem*{remark}{Remark}

\definecolor{darkgreen}{RGB}{0,140,0}

\title{
\vspace*{-0.5in}
{{\small \hfill TACL'26}\\
\vspace*{.25in}} 
Learning to Optimize Multi-Objective Alignment Through Dynamic Reward Weighting}

\author{
  \textbf{Yining Lu}$^\diamond$\thanks{\,\,\,Work done during an internship at Amazon.} \quad
  \textbf{Zilong Wang}$^\ddagger$\thanks{\,\,\,Work done at Amazon. Now at Google DeepMind. Email: \texttt{zlwang@ucsd.edu}} \quad
  \textbf{Shiyang Li}$^\ddagger$ \quad
  \textbf{Xin Liu}$^\ddagger$ \quad
  \textbf{Changlong Yu}$^\ddagger$ \\
  \textbf{Qingyu Yin}$^\ddagger$ \quad
  \textbf{Zhan Shi}$^\ddagger$ \quad
  \textbf{Zixuan Zhang}$^\ddagger$ \quad
  \textbf{Meng Jiang}$^\diamond$ \\
  \\
  $^\diamond$University of Notre Dame \\
  \texttt{\{ylu33, mjiang2\}@nd.edu} \\
  \\
  $^\ddagger$Amazon \\
  \texttt{\{syangli, xliucr, changlyu, qingyy, zzsamshi, zzx\}@amazon.com}
}

\date{}

\begin{document}
\maketitle
\begin{abstract}
Prior work in multi-objective reinforcement learning typically uses linear reward scalarization with fixed weights, which provably fails to capture non-convex Pareto fronts and thus yields suboptimal results. This limitation becomes especially critical in online preference alignment for large language models. Here, stochastic trajectories generated by parameterized policies create highly non-linear and non-convex mappings from parameters to objectives that no single static weighting scheme can find optimal trade-offs.
We address this limitation by introducing dynamic reward weighting, which adaptively adjusts reward weights during the online reinforcement learning process. Unlike existing approaches that rely on fixed-weight interpolation, our dynamic weighting continuously balances and prioritizes objectives in training, facilitating effective exploration of Pareto fronts in objective space.
We introduce two approaches of increasing sophistication and generalizability: hypervolume-guided weight adaptation and gradient-based weight optimization, offering a versatile toolkit for online multi-objective alignment. Our extensive experiments demonstrate their compatibility with commonly used online reinforcement learning algorithms, effectiveness across multiple datasets, and applicability to different model families, consistently achieving Pareto dominant solutions with fewer training steps than fixed-weight linear scalarization baselines.
\end{abstract}

\section{Introduction}
\begin{figure*}[ht]
    \centering
    \includegraphics[width=\linewidth]{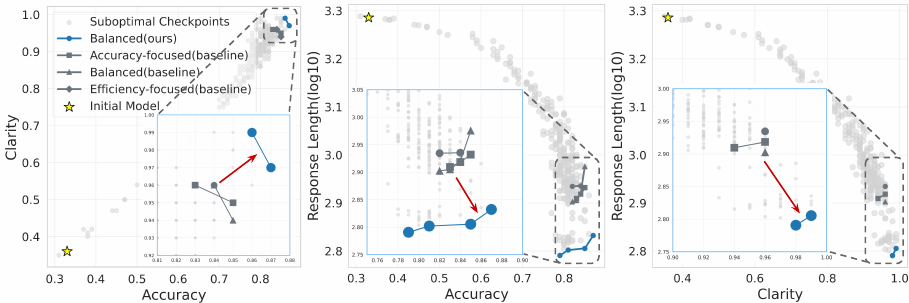}
    \caption{Pareto fronts obtained by our gradient-based weight optimization compared to three baselines using fixed-weight reward interpolation. We train the Qwen3-8B model \citep{yang2025qwen3technicalreport} on the Math500 dataset \citep{lightman2023lets} using GRPO \citep{shao2024deepseekmathpushinglimitsmathematical}. The three training configurations, \textit{accuracy-focused}, \textit{balanced}, and \textit{efficiency-focused}, correspond to different weight distributions initialized to our optimization objectives: \textit{accuracy}, \textit{conciseness}, and \textit{clarity}. We aim to train models that achieve strong problem-solving ability (higher accuracy) with computational efficiency (fewer tokens) while maintaining interpretable reasoning processes (better clarity). Gray dots indicate Pareto suboptimal checkpoints generated along training. \textbf{Clearly, our dynamic reward weighting consistently builds superior Pareto fronts that dominate baselines across all objectives, demonstrating its effectiveness in multi-objective alignment.}}
    \label{fig: teaser figure}
\end{figure*}

Online reinforcement learning (RL) has become the de facto approach in aligning large language models (LLMs) for complex reasoning tasks, such as mathematical problem solving \citep{shao2024deepseekmathpushinglimitsmathematical, zhang2025grpoleaddifficultyawarereinforcementlearning}, code generation \citep{chen2025acereasonnemotronadvancingmathcode, yao2025traininglanguagemodelsgenerate}, and logical reasoning \citep{xie2025logicrlunleashingllmreasoning, liu2025prorlprolongedreinforcementlearning}. While this approach has demonstrated significant success in improving model performance, it predominantly focuses on optimizing accuracy while overlooking other essential objectives that are crucial for practical deployment. For instance, key factors such as response length and output clarity, which directly impact inference efficiency and user experience, are not included in the rewarding process. In this paper, we focus on multi-objective online RL for LLMs that simultaneously optimizes auxiliary objectives alongside traditional accuracy metrics for reasoning tasks. 

Existing multi-objective alignment studies rely on either fixed weights \citep{yao2025traininglanguagemodelsgenerate, kimiteam2025kimik15scalingreinforcement} or heuristic rules for reward interpolation \citep{zhang2025grpoleaddifficultyawarereinforcementlearning, aggarwal2025l1controllinglongreasoning}. However, these methods have three critical limitations: (1) Empirically, we find that different objectives vary in learning difficulty.\footnote{We provide experimental justification in Appendix~\ref{appendix: motivation justification}.} As a result, objectives that reach saturation quickly will plateau throughout the remaining training phases while continuing to receive equal gradient updates, leading to inefficient allocation of learning efforts. (2) Theoretically, static linear scalarization provably fails to cover non-convex regions of the Pareto front and thus yields suboptimal training results \citep{10.5555/2591248.2591251, hayes2022practical}. (3) Heuristic rules for combining objectives lack flexibility and cannot generalize to new objectives or different tasks.

One natural way to address these limitations is to dynamically adjust objective weights throughout training based on the learning progress of each objective. This training paradigm, which we term \textit{dynamic reward weighting}, embodies a core principle that has proven effective in other optimization domains \citep{liu2024stochastic,doge}: \textit{Redirecting learning effort towards objectives with the greatest potential for improvement.}
In this work, we demonstrate that applying this principle to multi-objective LLM alignment could yield substantial improvements over existing static weighting schemes. As shown in \autoref{fig: teaser figure}, our gradient-based method outperforms all three baselines with varying weight configurations.

Specifically, we propose two progressively sophisticated approaches, spanning from strict to flexible training constraints, to suit different multi-objective alignment scenarios. (1) \textbf{Hypervolume-guided weight adaptation} (\S\ref{sec: vanilla reward adaptation}): When user preferences for different objectives are given, this method encourages the policy to discover new non-dominated solutions at each training step. It rewards new checkpoints that demonstrate positive hypervolume contributions, thereby proactively pushing the Pareto front in the desired optimization direction. (2) \textbf{Gradient-based weight optimization} (\S\ref{sec: online weight adaptation}): When user preferences are not available, the method computes how learning each objective contributes to improving overall model performance through gradient analysis and dynamically reallocates weights accordingly.  Extensive experiments demonstrate that our methods outperform static linear scalarization baselines and a concurrent work PAMA \citep{he2025paretomultiobjectivealignmentlanguage}, achieving superior Pareto fronts with fewer training steps across multiple online RL algorithms (GRPO, REINFORCE, and RLOO), datasets (Math500, MATH, and SafeSQL), and model families (Qwen3, Deepseek, Mistral, and Llama3). These results highlight the importance of dynamic reward weighting for multi-objective RL.\footnote{Code to reproduce our results: \href{https://github.com/yining610/dynamic-reward-weighting}{github.com/yining610/dynamic-reward-weighting}.} In summary, our contributions are threefold:
\begin{itemize}[leftmargin=*]
\item We identify the importance of dynamic reward weighting in multi-objective LLM alignment and formalize this optimization challenge.
\item We propose a comprehensive toolkit spanning hypervolume and gradient-based methods to address scenarios with and without human preference priors.
\item We demonstrate the effectiveness of our toolkit across multiple online RL algorithms, datasets, and models, showing improved training efficiency and Pareto optimal performance.
\end{itemize}

\section{Related Works}
Multi-objective RL seeks to balance multiple, sometimes conflicting objectives when optimizing policies \citep{manyobjectiveoptimization, emmerich2018tutorial, hayes2022practical, huang2025reinforcementlearningrubricanchors}. This is crucial because optimizing under each criterion can lead to significantly different policies being learning \citep{Rădulescu_Mannion_Zhang_Roijers_Nowé_2020}. We divide the notable prior works on multi-objective RL for LLMs into two groups: methods designed for training steerable multi-objective policies (\S\ref{subsec: steerable multi-objective finetuning}) and those focused on general reasoning preference alignment (\S\ref{subsec: general reasoning preference finetuning}).

\subsection{Steerable Multi-Objective Preference Finetuning}
\label{subsec: steerable multi-objective finetuning}
Most multi-objective RL works in the LLM era explore the problem from a steerability perspective, aiming to train policies that can be steered to generate desirable outputs across a continuum of user-preferred reward weights. Representative approaches include off-policy RL methods \citep{zhou-etal-2024-beyond, guo-etal-2024-controllable, xiong2025projectionoptimizationgeneralframework}, post-hoc policy mixing \citep{NEURIPS2023_e12a3b98, wang2025mpoefficientpostprocessingframework}, and conditional training that incorporates preference weights as input \citep{NEURIPS2024_3d038008, yang2024rewards, wang-etal-2024-conditional, NEURIPS2024_89f39d0b}.

However, these methods typically treat objectives in a static manner, either by using fixed-weight linear scalarization for rewards or by merging trained policies with fixed weights. While they enable inference-time flexibility for users to navigate trade-offs between objectives, they suffer from a fundamental limitation: linear scalarization provably fails to capture non-convex regions of the Pareto front \citep{hayes2022practical, 10.5555/2591248.2591251}, resulting in suboptimal policies. 

From a geometric perspective, linear scalarization is equivalent to sweeping a hyperplane whose normal vector is the given weights over the objective space. According to the supporting hyperplane theorem \citep{boyd2004convex}, for each weight vector, there exists a non-zero supporting hyperplane that is tangent to the convex hull of the Pareto front at some point. Consequently, linear scalarization can only recover Pareto optimal solutions on the convex portions of the Pareto front but fails on those in non-convex (concave) regions \citep{wei2022fairness}.

This limitation becomes particularly critical in online RL settings, where stochastic trajectories create highly non-linear and non-convex mappings from policy parameters to objectives. We therefore replace static weighting with a dynamic reward weighting mechanism that continuously rebalances and reprioritizes objectives during training. This approach facilitates a more thorough exploration of the objective space, enabling the final policy to approximate Pareto optimal solutions in regions that static linear methods cannot reach.

\subsection{General Multi-Objective Preference Finetuning}
\label{subsec: general reasoning preference finetuning}
A more straightforward application of multi-objective RL is to train a policy to acquire multiple capabilities in a single training run \citep{mukherjee2024multiobjectivealignmentlargelanguage}, which is the direction we aim to improve. Since the introduction of PPO \citep{ouyang2022traininglanguagemodelsfollow} and GRPO \citep{shao2024deepseekmathpushinglimitsmathematical}, on-policy RL has become the de facto approach for LLM preference alignment. \citet{NEURIPS2023_b8c90b65} have applied RLHF to train LLMs that generate more relevant, factual, and informative text. Beyond accuracy, GRPO has been used to train LLMs for high-quality code generation \citep{yao2025traininglanguagemodelsgenerate} and to produce more concise responses in mathematical reasoning tasks \citep{zhang2025grpoleaddifficultyawarereinforcementlearning, kimiteam2025kimik15scalingreinforcement, aggarwal2025l1controllinglongreasoning}. Yet these methods still compute rewards with static linear scalarization or rely on human-defined interpolation rules, limiting their generalizability to new objectives.

To the best of our knowledge, no unified toolkit currently exists for automatically guiding LLM training across arbitrary objective sets. The most relevant line of work is the multi-gradient descent algorithm (MGDA) for multi-objective optimization \citet{DESIDERI2012313, NEURIPS2022_f91bd64a, liu2024stochastic, gu2024safebalancedframeworkconstrained, Xu_Ju_Liu_Yang_2025, li2025gradientadaptivepolicyoptimizationmultiobjective}. These methods compute gradients for each objective and combine them using various aggregation strategies to guide parameter updates. However, existing approaches primarily focus on relatively simple supervised fine-tuning tasks, such as logistic regression, and study on classical two-objective Fonseca problems \citep{fonseca}. They are not directly applicable to LLM alignment. 

Most closely related to our work, the concurrent study PAMA \citep{he2025paretomultiobjectivealignmentlanguage} adopts the minimum-norm optimization principle from MGDA to LLM alignment. While PAMA represents an important step toward multi-objective LLM training, it has only been evaluated on general preference alignment tasks with two objectives (e.g., generating positive sentiment and longer responses on the IMDb dataset \citep{maas-etal-2011-learning}) using Noon PPO. Its generalizability to more complex reasoning tasks and to different RL algorithms remains unclear.
\section{Background}
\subsection{Notations}
We now formulate common notations used throughout the paper. We consider a multi-objective optimization problem with $K$ objectives, where each objective is associated with a weight $w_i$ for $i = 1, 2, \ldots, K$. The weight vector is denoted as $\vw = (w_1, w_2, \ldots, w_K)$, and the corresponding reward vector is $\vr = (r_1, r_2, \ldots, r_K)$.
Following standard practice in RL, we denote the policy parameterized by $\theta$ as $\pi_\theta$, and the objective function as $J(\theta)$. We use superscript $w$ to denote weighted quantities (e.g., $G^w_t$ for weighted discounted return and $r^w$ for weighted reward) and $i$ to denote single-objective quantities (e.g., $w_i$ and $J_i(\theta)$).
\subsection{Pareto Front and Hypervolume Contribution}
\begin{definition}[Pareto Front]
Let $X\subseteq\mathbb{R}^n$ be the feasible set and
$f:X\to\mathbb{R}^K$, $f(x)=(f_1(x),\ldots,f_K(x))$, the vector of
objectives \textit{to be maximized}.  
Write $u\succ v$ if $u_i\geq v_i$ for every $i$ and $u\neq v$. The Pareto set and Pareto front are formally defined as:
\begin{align*}
&\text{Pareto set} \coloneqq P = \left\{ x \in X \mid \nexists y \in X : f(y) \succ f(x) \right\}\\ &\text{Pareto front} \coloneqq f(P) = \left\{ f(x) \mid x \in P \right\}.
\end{align*}
Intuitively, the Pareto front consists of all objective vectors for which no objective can be improved without sacrificing at least one other objective. Points in the Pareto set are called \textit{Pareto optimal} or \textit{non-dominated} solutions.
\end{definition}

\begin{definition}[Hypervolume Indicator]
Let $\mA = \{\va^{(1)}, \ldots, \va^{(n)}\} \subset \mathbb{R}^K$ be a finite set of objective vectors, where each $\va^{(i)} = (a^{(i)}_1, \ldots, a^{(i)}_K)$. Let $\vr = (r_1, \ldots, r_K) \in \mathbb{R}^K$ be a reference point satisfying $r_j \leq a^{(i)}_j$ for all $i \in \{1, \ldots, n\}$ and $j \in \{1,\ldots, K\}$ (i.e., $\vr$ is dominated by every $\va^{(i)}$). The hypervolume of $\mA$ with respect to $\vr$ is:
\begin{align*}
\text{HV}(\mA;\vr)=
\Lambda\Bigl(\bigcup_{\va\in\mA} [\va,\vr]\Bigr), \;\text{where }
[\va,\vr] = \\ \{\,x\in\mathbb{R}^K \mid r_j \le x_j \le a_j,\; \forall j \in \{1, \ldots, K\}\}
\end{align*}
$\Lambda$ denotes the $K$-dimensional Lebesgue measure. The hypervolume indicator is the $K$-dimensional volume of the region dominated by $\mA$ and bounded below by $\vr$; higher hypervolume indicates larger coverage of the objective space.
\end{definition}

\begin{definition}[Hypervolume Contribution]
For any point $\va\in \mathbb{R}^K$, its hypervolume contribution to a set $\mA$ is defined as the change in hypervolume when $\va$ is presented and therefore it is always non-negative:
$$
\Delta\text{HV}(\va,\mA) = \text{HV}(\mA \cup \{\va\}) -\text{HV}\bigl(\mA\setminus\{\va\}) \geq 0.
$$
\end{definition}

\section{Hypervolume-Guided Weight Adaptation}
\label{sec: vanilla reward adaptation}

A natural and straightforward way to improve multi-objective alignment is to encourage pushing the Pareto front at each training step. Therefore, we introduce hypervolume-guided weight adaptation, which leverages human-specified weights $\vw$ through two hierarchical components: (1) the standard reward vector $\vr$, and (2) a meta-level reward $r_{\text{pareto}}(\vr, \mB)$ that encourages policies to achieve new Pareto optimality. Note that we do not directly modify the given human-specified weights $\vw$, instead we use the meta-level signal to amplify rewards when training discovers new Pareto fronts. We design $r_{\text{pareto}}(\vr)$ as:\footnote{We experimented with multiple other activation functions with varying range and found that Eq.~\ref{eq: r pareto} is the most stable one for our scenario. We provide more details in Appendix~\ref{appendix: hyperparameters}.}
\begin{equation}
r_{\text{pareto}}(\vr, \mB) = 0.5 + 1.5\tanh(\Delta\text{HV}(\vr, \mB)).
\label{eq: r pareto}
\end{equation}
$\mB$ represents the performance buffer storing validation performance of the current Pareto set. In practice, we compute $\Delta\text{HV}(\vr, \mB)$ using the recursive dimension-sweep algorithm of \citet{1688440}, which achieves computational complexity $\mathcal{O}(n^{k-2} \log n)$ for $n$ points in $k$ dimensions. The complete procedure is detailed in Algorithm~\ref{alg: vanilla}.

\subsection{Experiment Setup}
\label{subsec: vanilla experiment setup}
\begin{algorithm}[ht]
    \small
    \captionsetup{type=algorithm}
    \captionof{algorithm}{Hypervolume-Guided Weight Adaptation}
    \label{alg: vanilla}
    \begin{algorithmic}[1]
      \State \textbf{Input:} Training set $\mathcal{D}_{\text{train}}$, validation set $\mathcal{D}_{\text{val}}$, initial policy parameters $\theta_{0}$, human specified weights $\vw\!\in\!\mathbb{R}^{K}_{\ge 0}$
      \State \textbf{Hyperparameters:} batch size $B$, rollout size $G$, maximum training steps $T$
      \State Evaluate $\vr_{\theta_{0}}\!=\!\text{Evaluate}(\theta_{0},\mathcal{D}_{\text{val}})$
      \State Initialize Pareto set $\mB\leftarrow\{\vr_{\theta_{0}}\}$
      \State $r_{\text{pareto}}\leftarrow 1$ \Comment{\textcolor{darkgreen}{No hypervolume contribution yet}}
      \For{$t\!=\!1$ \textbf{to} $T$}
        \State Sample mini-batch $\mathcal{D}_{b}\subset\mathcal{D}_{\text{train}}$ of size $B$
        \ForAll{query $q\in\mathcal{D}_{b}$}
          \State Sample $G$ answers $\{y_{i}\}^{G}_{i=1}\sim\pi_{\theta_{t-1}}(\cdot\mid q)$
          \State Compute reward vectors $\{\vr_{i}\}^{G}_{i=1}$ for each $y_{i}$
          \State Compute scalar rewards $r_{i}\gets\vw^{\top}\vr_{i}$
          \State \textbf{Adaptation:} $\tilde{r}_{i}\gets r_{\text{pareto}}\cdot r_{i}$
        \EndFor
        \State Update $\theta_{t-1}$ via chosen RL objective function $J(\theta_{t-1})$ using $\tilde{r}_i$
        \State $\vr_{\theta_{t}}\gets\text{Evaluate}(\theta_{t},\mathcal{D}_{\text{val}})$
        \State $r_{\text{pareto}}\gets 0.5 + 1.5\tanh(\Delta\text{HV}(\vr_{\theta_t}, \mB))$ \Comment{\textcolor{darkgreen}{Update meta weight based on hypervolume contribution}}
        \If{$\Delta\text{HV}(\vr_{\theta_{t}};\mB)>0$}
          \State $\mB\leftarrow\mB\cup\{\vr_{\theta_{t}}\}$ \Comment{\textcolor{darkgreen}{Update Pareto set}}
        \EndIf
      \EndFor
    \end{algorithmic}
\end{algorithm}

\begin{table*}[ht]
\setlength{\belowcaptionskip}{-5pt}
\centering
\small
\setlength{\tabcolsep}{2pt}
\renewcommand{\arraystretch}{0.92}
\newcolumntype{G}{!{\color{black!35}\vrule width 0.6pt}} 
\begin{adjustbox}{width=\linewidth}
\begin{tabular}{@{}l c G c G c @{\hskip 6pt} c G c @{\hskip 6pt} c G c @{}}
\toprule
\multirow{2}{*}{\textbf{Method}}
  & \multicolumn{1}{c}{\textbf{PAMA}}
  & \multicolumn{2}{c}{\textbf{Accuracy-focused}}
  & \multicolumn{2}{c}{\textbf{Balanced}}
  & \multicolumn{2}{c}{\textbf{Efficiency-focused}} \\
\cmidrule(lr){2-2}\cmidrule(lr){3-4}\cmidrule(lr){5-6}\cmidrule(lr){7-8}
  & \multicolumn{1}{c}{ Baseline}
  & \multicolumn{1}{c}{Baseline} & \multicolumn{1}{c}{Hypervolume-guided}
  & \multicolumn{1}{c}{Baseline} & \multicolumn{1}{c}{Hypervolume-guided}
  & \multicolumn{1}{c}{Baseline} & \multicolumn{1}{c}{Hypervolume-guided} \\
\textbf{Training RL}
  & \multicolumn{7}{c}{\textbf{\emph{Accuracy $\uparrow$ \, / \, Response Length $\downarrow$ \, / \, Clarity $\uparrow$}}} \\
\midrule
GRPO       & 0.360 / 1936 / 0.385 & 0.832 / 701 / 0.962 & \textbf{0.850} / \textbf{619} / \textbf{0.970}  & 0.832 / 687 / 0.935 & 0.825 / \textbf{683} / \textbf{0.955} & 0.837 / 650 / 0.970 & \textbf{0.840} / 731 / 0.967 \\
REINFORCE  & 0.380 / 1921 / 0.400 & 0.789 / 837 / 0.977 & \textbf{0.805} / \textbf{827} / \textbf{0.988}  & 0.798 / 872 / 0.994 & 0.788 / \textbf{837} / 0.978 & 0.778 / 676 / 0.985 & \textbf{0.790} / \textbf{618} / \textbf{1.000}  \\
RLOO       & 0.340 / 1939 / 0.365 & 0.827 / 677 / 0.965 & \textbf{0.829} / 813 / 0.940  & 0.820 / 720 / 0.940 & \textbf{0.823} / \textbf{639} / \textbf{0.955} & 0.830 / 565 / 0.967 & \textbf{0.843} / 700 / \textbf{0.980} \\
\bottomrule
\end{tabular}
\end{adjustbox}
\vspace{-2mm}
\caption{Pareto fronts obtained from hypervolume-guided weight adaptation compared to PAMA and fixed-weight baselines. $\uparrow$ and $\downarrow$ indicate optimization direction. Each value represents the average performance across the entire Pareto front under a given training setup. \textbf{Hypervolume-guided weighting outperforms baselines across most objectives, weight configurations, and RL algorithms}, and in certain cases, demonstrates full dominance with superior results on all three objectives.}
\label{table: vanilla}
\end{table*}

We evaluate our proposed method across multiple online RL algorithms, including natural policy gradient methods REINFORCE \citep{NIPS1999_464d828b} and RLOO \citep{ahmadian2024basicsrevisitingreinforcestyle}, and heuristic policy gradient method GRPO \citep{shao2024deepseekmathpushinglimitsmathematical}. The main experiments are conducted on the Math500 dataset \citep{lightman2023lets} using the Qwen3-8B model \citep{yang2025qwen3technicalreport}. To provide a more comprehensive evaluation of our method, we then extend our experiments to additional datasets spanning mathematical and coding reasoning tasks (Math algebra problems \citep{mathlighteval} and SafeSQL \citep{yao2025traininglanguagemodelsgenerate}), as well as additional model families, including Deepseek-7B\citep{deepseekai2024deepseekllmscalingopensource}, Mistral-7B \citep{jiang2023mistral7b}, and Llama3-8B \citep{grattafiori2024llama3herdmodels}, in \S\ref{sec: extensive experiment}. All results are reported from the test set.

For math tasks, we design the following three objectives: (1) generating correct solutions (\textit{accuracy}), (2) producing efficient solutions that minimize computational steps and tokens (\textit{conciseness}), and (3) demonstrating clear reasoning with step-by-step thinking process (\textit{clarity}). We denote the weight vector as $\vw = [w_\text{accuracy},w_\text{conciseness},w_\text{clarity}]$. Each objective is evaluated using heuristic rules to ensure the corresponding reward is verifiable (0 or 1). For easier analysis, in the following tables and figures, we present accuracy and clarity performance based on their reward scores, while conciseness is reported directly through response length.

Following common practice in multi-objective RL \citep{NEURIPS2023_e12a3b98, li2025multiobjectivelargelanguagemodel}, we use fixed-weight linear scalarization as one of our baselines. Specifically, we evaluate three carefully chosen weight configurations that represent different optimization priorities: 1) \textit{Accuracy-focused} ($\vw = [0.5, 0.25, 0.25]$) prioritizing accuracy over conciseness and formatness; 2) \textit{Balanced} ($\vw = [0.334, 0.333, 0.333]$) weighting all objectives equally; and 3) \textit{Efficiency-focused} ($\vw = [0.25, 0.375, 0.375]$) emphasizing conciseness and clarity. Additionally, we compare our method against the concurrent work PAMA \citep{he2025paretomultiobjectivealignmentlanguage}.\footnote{As their code is not publicly available, we reimplement PAMA following the methodology and hyperparameters reported in the paper.} More experiment setup details can be found in Appendix~\ref{appendix: vanilla details}.

\subsection{Results}
We compare our hypervolume-guided weight adaptation against fixed-weight baselines and PAMA in \autoref{table: vanilla}. Across all three online RL algorithms, there is consistently at least one weight configuration where our method outperforms the baselines on all objectives. For example, under both accuracy- and efficiency-focused settings, REINFORCE equipped with hypervolume-guided weighting achieves Pareto fronts that have better average scores across all three objectives.

It is important to note that better average scores of Pareto fronts do not necessarily guarantee true Pareto optimality. Due to space constraints, we include the full Pareto front visualizations for \autoref{table: vanilla} in Appendix~\ref{appendix: visualization} (\autoref{fig: reinforce visualization}). These plots provide compelling visual evidence that REINFORCE with our hypervolume-guided weight adaptation achieves clearly superior Pareto fronts, dominating baselines under both accuracy- and efficiency-focused settings.

To better understand the performance differences between weight configurations, we analyze the meta-reward $r_{\text{pareto}}$ during REINFORCE training. As shown in \autoref{fig: meta reward distribution}, the balanced weight setting yields higher meta-rewards that encourage more aggressive exploration, paradoxically, leading to suboptimal performance compared to two other configurations. This suggests that effective hypervolume-guided weighting requires careful calibration of meta-reward to avoid overly aggressive updates that can hinder convergence.
\begin{figure}[t]
  \centering
  \includegraphics[width=\linewidth]{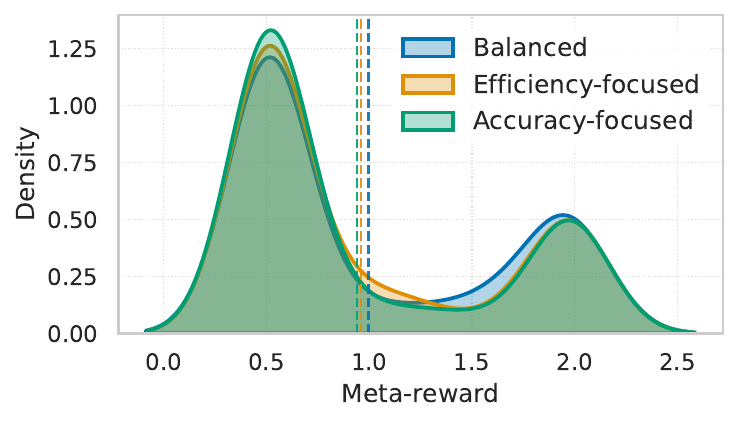}
  \captionsetup{aboveskip=2pt, belowskip=0pt}
  \caption{Meta-reward $r_{\text{pareto}}$ distributions, with vertical dashed lines indicating the average value.}
  \label{fig: meta reward distribution}
\end{figure}

\section{Gradient-Based Weight Optimization}
\label{sec: online weight adaptation}
We now consider a more sophisticated yet flexible setting where reward weights are not fixed but dynamically adapted in real-time. In this scenario, human-specified preferences are not assumed. We formulate the online weights adaptation during training as an optimization problem and begin our study with the natural policy-gradient method REINFORCE:
\begin{align}
\nabla J(\theta)=\mathbb{E}\!\big[\nabla\log\pi_{\theta}(a_{t}\mid s_{t})\,G^{w}_{t}\big],
\label{eq: reinforce}
\end{align}
which exhibits linearity to the per-objective gradients given by $\nabla J(\theta)=\sum_{i=1}^{K} w_{i}\,\nabla J_i(\theta)$. We prove this linearity through a simple nested summation 
\begin{align*}
G^w_t &= \sum_{l=0}^{\infty}\gamma^lr_{t+l+1}^w = \sum_{l=0}^{\infty}\gamma^l \sum_{i=1}^Kw_i r_{t+l+1}^i \\
&= \sum_{i=1}^K w_i\sum_{l=0}^{\infty}\gamma^lr_{t+l+1}^i = \sum_{i=1}^Kw_i G_{t}^i.
\end{align*} 
Then there exists a constant $C \in \R$ such that $J(\theta) = \sum_{i=1}^Kw_i J_i(\theta) + C$, which implies
$$
\argmin_{\vw \in \Delta^K} J(\theta(\vw)) = \argmin_{\vw \in \Delta^K}\sum_{i=1}^Kw_i J_i(\theta).
$$
Inspired by domain adaptation \citep{doge} and based on the above linearity assumption, we can derive the following reward weight update rule, where $\eta$ is the learning rate and $\mu$ is the regularization factor (we provide full proof in Appendix \ref{appendix: proof} and the whole procedure in Algorithm~\ref{alg: optimization}):
\begin{align}
& \vw^{(t)} = \frac{\vw^{(t)}}{\sum_{k}w_k^{(t)}}, \;\vw^{(t)} = \vw^{(t-1)} \odot \exp(\frac{\eta^{(t)}\mI^{(t)}}{\mu}), \nonumber \\ 
&I_i^{(t)} = \langle \nabla J_i(\theta^{(t)}), \sum_{k\in [K]} \nabla J_k(\theta^{(t)}) \rangle.
\label{eq: weight optimization}
\end{align}
\begin{algorithm}[ht]
    \small
    \captionsetup{type=algorithm}
    \captionof{algorithm}{\textbf{Online Reward Weights Optimization}}
    \label{alg: optimization}
    \begin{algorithmic}[1]
      \State \textbf{Input:} Training set $\mathcal{D}_{\text{train}}$, initial policy parameters $\theta_{0}$, initial objective weights $\vw^{(0)}\!\in\!\mathbb{R}^{K}_{\ge 0}$
      \State \textbf{Hyperparameters:} batch size $B$, rollout size $G$, learning rate $\eta$, regularization factor $\mu$, maximum training steps $T$
      \For{$t\!=\!1$ \textbf{to} $T$}
        \State Sample mini-batch $\mathcal{D}_{b}\subset\mathcal{D}_{\text{train}}$ of size $B$
        \ForAll{query $q\in\mathcal{D}_{b}$}
          \State Sample $G$ answers $\{y_{i}\}_{i=1}^{G}\sim\pi_{\theta_{t-1}}(\cdot\mid q)$
          \State Compute reward vectors $\{\vr_{i}\}_{i=1}^{G}$ for each $y_{i}$
        \EndFor
        \State Obtain gradients $\{\nabla J_{i}(\theta_{t-1}, \mathcal{D}_{b})\}_{i=1}^{K}$
        \State Compute influence signal $\mI^{(t)}$ \Comment{\textcolor{darkgreen}{See Eq.~\ref{eq: weight optimization}}}
        \State \textbf{Update weights:}
        \State \hspace{1em} $\vw^{(t)} \leftarrow \vw^{(t-1)} \odot \exp\!\big(\tfrac{\eta^{(t)}\,\mI^{(t)}}{\mu}\big)$
        \State \hspace{1em} $\vw^{(t)} \leftarrow \vw^{(t)} \big/ \sum_{i} w^{(t)}_{i}$
        \State Compute scalar rewards $\{\,\vw^{(t)\!\top}\vr_{i}\,\}_{i=1}^{G}$
        \State Update $\theta_{t-1}$ via chosen RL objective function $J(\theta_{t-1})$ using scalarized rewards
      \EndFor
    \end{algorithmic}
\end{algorithm}

\begin{remark}
The $I_i^{(t)}$ measures the total influence that learning objective $i$ has on the remaining $K-1$ objectives ($\langle \nabla J_i, \sum_{k\neq i} \nabla J_k\rangle $) plus the magnitude of its own gradient ($\|\nabla J_i\|_2^2$). Intuitively, similar to prior studies in data influence \citep{pmlr-v70-koh17a} and domain adaptation \citep{doge}, we upweight an objective when it has high learning potential on other objectives (high influence) and has not been learned enough yet (high gradient magnitude).
\end{remark}

\subsection{Convergence Analysis}
A natural concern in weight updating is whether the weights may collapse to zero or explode to the boundary value of one within a single update step. To address this concern, we provide a convergence analysis for Eq.~\ref{eq: weight optimization}.
We first make some standard assumptions used across existing RL \citep{NIPS2017_353de269, NEURIPS2020_cc9b3c69, 10.1007/s10994-023-06303-2} and multi-objective optimization studies \citep{NEURIPS2022_f91bd64a, liu2024stochastic, Xu_Ju_Liu_Yang_2025}.

\begin{assumption}[Lipschitz Continuity]
    \label{ass: lipschitz continuity}
    The policy $\pi_\theta$ is differentiable with respect to $\theta$ and the output $\log\pi_\theta(a\mid s)$ is L-Lipschitz and has bounded norm $\|\nabla \log\pi_\theta(a\mid s)\| \leq B$ for any $\theta$.
\end{assumption}

\begin{assumption}[Bounded Reward]
    \label{ass: bounded reward}
    The reward is bounded in the range $[0, r_{\text{max}}]$.
\end{assumption}
\begin{assumption}[Learning Rate]
    \label{ass: bounded learning rate}
    The learning rate $\eta^{(t)}$ is non-increasing and non-negative.
\end{assumption}

\begin{assumption}[Bounded Policy Gradient]
\label{ass: bounded policy gradient}
The true gradient (Eq.~\ref{eq: reinforce}) of the objective function is bounded $\|\nabla J_k(\theta)\| \leq C \leq \infty$ which is established by $\|\nabla \log\pi_{\theta}(a\mid s) \| < B$ (Assumption~\ref{ass: lipschitz continuity}) and $G_t = \sum_{l=0}^{\infty} \gamma^l r_{t+l+1} \leq \sum_{l=0}^{\infty} \gamma^l r_{\text{max}} = \frac{r_{\text{max}}}{1-\gamma}$ (Assumption~\ref{ass: bounded reward}). 
\end{assumption}

\begin{lemma}
Recall the weight update rule defined in Eq.~\ref{eq: weight optimization}, its true weight takes the closed form:
\begin{align*}
w_i^{(T)} = \frac{w^{(0)}_i\exp\big(\sum_{t=1}^T \tau^{(t)}I_i^{(t)} \big)}{\sum_{k\in [K]}w_k^{(0)}\exp\big(\sum_{t=1}^T \tau^{(t)} I_k^{(t)} \big)},\\
\text{where }\tau^{(t)} = \frac{\eta^{(t)}}{\mu} 
\end{align*}
\end{lemma}
\begin{proof}
The lemma holds when $T=1$: 
$$
w_i^{(1)} = \frac{w^{(0)}_i\exp\big( \tau^{(1)}I_i^{(1)} \big)}{\sum_{k\in [K]}w_k^{(0)}\exp\big(\tau^{(1)} I_k^{(1)} \big)}.
$$
Suppose the lemma holds for all values up to $T-1$, for step at $T$, we have:
\begin{align*}
    &w_i^{(T)} = \frac{w_i^{(T-1)}\exp(\tau^{(T)}I_i^{(T)})}{\sum_k w_k^{(T-1)}\exp(\tau^{(T)}I_k^{(T)})} \\
    &= \frac{\frac{w^{(0)}_i\exp\big(\sum_{t=1}^{T-1} \tau^{(t)}I_i^{(t)} \big)}{\sum_{k\in [K]}w_k^{(0)}\exp\big(\sum_{t=1}^{T-1} \tau^{(t)} I_k^{(t)} \big)}\exp(\tau^{(T)}I_i^{(T)})}{\sum_k\frac{w^{(0)}_k\exp\big(\sum_{t=1}^{T-1} \tau^{(t)}I_k^{(t)} \big)}{\sum_{k\in [K]}w_k^{(0)}\exp\big(\sum_{t=1}^{T-1} \tau^{(t)} I_k^{(t)} \big)} \exp(\tau^{(T)}I_k^{(T)})} \\
    &= \frac{w^{(0)}_i\exp\big(\sum_{t=1}^{T-1} \tau^{(t)}I_i^{(t)} \big) \exp(\tau^{(T)}I_i^{(T)})}{\sum_k w^{(0)}_k\exp\big(\sum_{t=1}^{T-1} \tau^{(t)}I_k^{(t)} \big) \exp(\tau^{(T)}I_k^{(T)})} \\
    &= \frac{w^{(0)}_i\exp\big(\sum_{t=1}^T \tau^{(t)}I_i^{(t)} \big)}{\sum_{k}w_k^{(0)}\exp\big(\sum_{t=1}^T \tau^{(t)} I_k^{(t)} \big)}
\end{align*}
\end{proof}

\begin{theorem}
Suppose Assumptions~\ref{ass: lipschitz continuity}-\ref{ass: bounded policy gradient} hold and that the step-size sequence of $\tau$ converges $|\sum_{t=1} \tau^{(t)} - \ell| < \epsilon$ for every arbitrarily small positive number $\epsilon$. Then, for every pair of objectives $i, j \in [K]$, their weight ratio is uniformly bounded: $\smash{w_i^{(T)}/w_j^{(T)}} = \mathcal{O}(1), \;\forall T\in \sN$.
\label{theorem: convergence analysis}
\end{theorem}

\begin{table*}[ht]
\centering
\small
\setlength{\tabcolsep}{2pt}
\setlength{\belowcaptionskip}{-5pt}
\renewcommand{\arraystretch}{0.92}
\begin{adjustbox}{width=\linewidth}
\begin{tabular}{@{}l c c c c c @{}}
\toprule
\textbf{Method $\rightarrow$}
  & \multicolumn{1}{c}{\textbf{PAMA}}
  & \multicolumn{1}{c}{\textbf{Accuracy-focused}}
  & \multicolumn{1}{c}{\textbf{Balanced}}
  & \multicolumn{1}{c}{\textbf{Efficiency-focused}}
  & \multicolumn{1}{c}{\textbf{Gradient-based(Ours)}} \\
\cmidrule(lr){2-2}\cmidrule(lr){3-3}\cmidrule(lr){4-4}\cmidrule(lr){5-5}\cmidrule(lr){6-6}
\textbf{Training RL $\downarrow$}
  & \multicolumn{5}{c}{\textbf{\emph{Accuracy $\uparrow$ \, / \, Response Length $\downarrow$ \, / \, Clarity $\uparrow$}}} \\
\midrule
GRPO
  & 0.360 / 1934 / 0.387
  & 0.836 / 831 / 0.948
  & 0.830 / 861 / 0.960
  & 0.835 / 842 / 0.950
  & \textbf{0.836} / \textbf{650} / \textbf{0.980} \\
REINFORCE
  & 0.341 / 1929 / 0.369
  & 0.764 / 1375 / 0.830
  & 0.760 / 1336 / 0.850
  & 0.755 / 1361 / 0.825
  & \textbf{0.802} / \textbf{1202} / \textbf{0.868} \\
RLOO
  & 0.355 / 1938 / 0.385
  & 0.815 / 1100 / 0.905
  & 0.820 / 847 / 0.937
  & 0.830 / 824 / 0.940
  & 0.820 / \textbf{701} / \textbf{0.980} \\
\bottomrule
\end{tabular}
\end{adjustbox}
\vspace{-2mm}
\caption{Pareto fronts obtained from gradient-based weight optimization compared to PAMA and fixed-weight baselines. $\uparrow$ and $\downarrow$ indicate optimization direction. Each value represents the average performance across the entire Pareto front under a given training setup. \textbf{Clearly, the gradient-based weighting approach yields superior Pareto fronts compared to all baselines across both GRPO and REINFORCE training setups}, and achieves improved performance on two key objectives (conciseness and clarity) under RLOO training. }
\label{table: optimization}
\end{table*}

\begin{proof}
\begin{align*}
    \frac{w_i^{(T)}}{w_j^{(T)}} &= \frac{w_i^{(0)}\exp\big(\sum_{t=1}^T \tau^{(t)}I_i^{(t)} \big)}{w_j^{(0)}\exp\big(\sum_{t=1}^T \tau^{(t)}I_j^{(t)} \big)}  \\
    &= \frac{w_i^{(0)}}{w_j^{(0)}}\exp\Big(\sum_{t=1}^T \tau^{(t)}(I_i^{(t)} - I_j^{(t)})\Big) \\
    &\leq \frac{w_i^{(0)}}{w_j^{(0)}} \exp\Big(\sum_{t=1}^T \tau^{(t)}2KC^2\Big) \\
    &\leq \frac{w_i^{(0)}}{w_j^{(0)}}\exp(2KC^2\ell) = \mathcal{O}(1)
\end{align*}
\end{proof}
\begin{remark}
The ratio $\smash{w_i^{(T)}/w_j^{(T)}}$ measures the cumulative advantage that objective $i$ has gained over objective $j$ throughout training, determined by both their initial weights $w^{(0)}$ and the accumulated influence across iterations $I^{(t)}$. Intuitively, an objective that is more influential on the gradient update will naturally gain a higher weight relative to a less influential objective.
Keeping this ratio uniformly bounded has two merits. First, it offers numerical stability. A bounded ratio prevents weight collapse or explosion from the exponential update, ensuring that the updated weights $\vw$ remain properly constrained and ultimately converge to a valid multi‑objective equilibrium. Second, it has practical guidance for hyperparameter selection. Initializing $\vw$ uniformly and employing a learning rate scheduler in which $\eta^{(t)}$ converges, so that each step updates the weights within safe and predictable bounds.
\end{remark}

\subsection{Experiment Setup}
As suggested by Theorem~\ref{theorem: convergence analysis}, we initialize the reward weights using a balanced (uniform) setting and use a polynomial learning rate scheduler whose learning rate sequence converges, instead of the constant scheduler used in \S\ref{sec: vanilla reward adaptation}.
To maintain the linearity between the overall reward and the per-objective rewards, we do not add the token-level KL divergence penalty from the reference model in each reward. Furthermore, standard approaches for preference finetuning LLMs with REINFORCE, such as REINFORCE++ \citep{hu2025reinforceefficientrlhfalgorithm}, typically employ clipping mechanisms that render the objective non-differentiable and break the linearity assumptions made in Eq.~\ref{eq: reinforce}. Therefore, in our experiments, we set both the clip range $\epsilon$ \citep{schulman2017proximalpolicyoptimizationalgorithms} and dual clip constant $c$ \citep{Ye_Liu_Sun_Shi_Zhao_Wu_Yu_Yang_Wu_Guo_Chen_Yin_Zhang_Shi_Wang_Fu_Yang_Huang_2020} high ($\epsilon=c=100$) to practically disable clipping.\footnote{We also experimented with standard clipping settings ($\epsilon=0.2, c=3$) and observed that disabling clipping did lead to improved training results.} To improve computational efficiency, we compute gradients only from middle layers, as these layers typically encode the most abstract and generalizable knowledge patterns \citep{grosse2023studyinglargelanguagemodel}. All other experiment settings follow \S\ref{subsec: vanilla experiment setup}.

\subsection{Results}

\begin{figure}[ht]
  \centering
  \includegraphics[width=\linewidth]{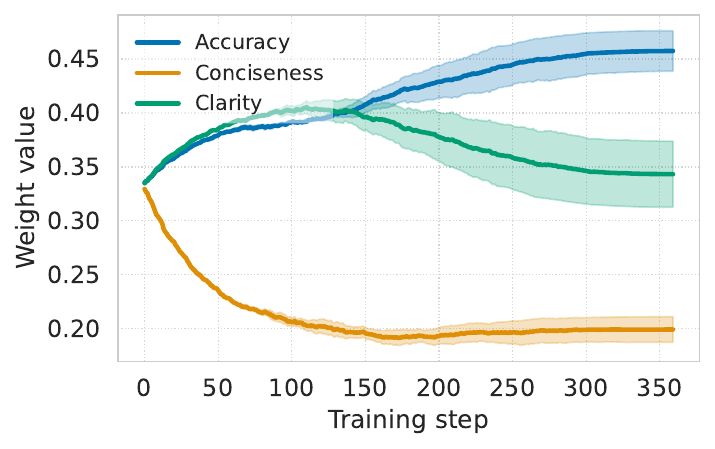}
  \captionsetup{aboveskip=2pt, belowskip=0pt}
  \caption{Reward weight evolution over training.}
  \label{fig: weight evolution}
\end{figure}
\begin{figure*}[ht]
  \setlength{\belowcaptionskip}{-3pt}
  \centering

  \captionsetup[subfigure]{skip=1pt, belowskip=8pt}

  \begin{subfigure}[t]{\textwidth}
    \centering
    \begin{subfigure}[t]{0.333\textwidth}
      \includegraphics[width=\linewidth]{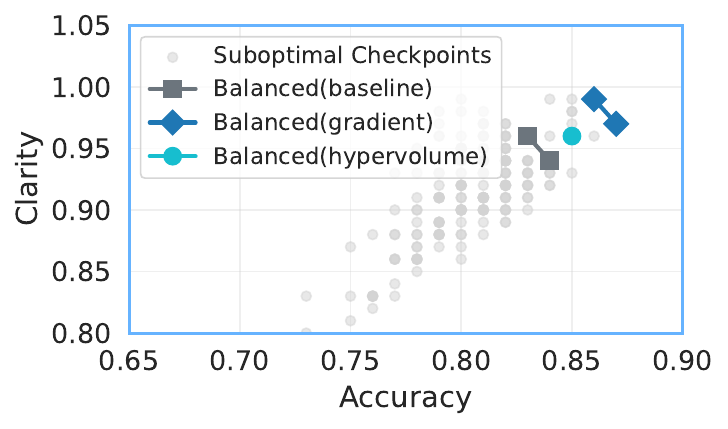}
    \end{subfigure}%
    \begin{subfigure}[t]{0.333\textwidth}
      \includegraphics[width=\linewidth]{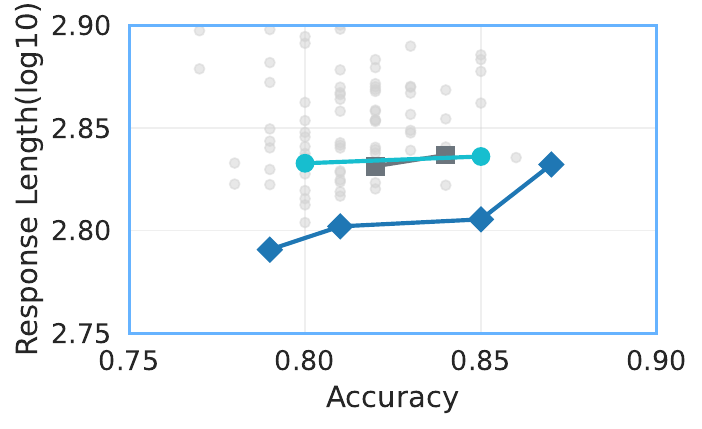}
    \end{subfigure}%
    \begin{subfigure}[t]{0.333\textwidth}
      \includegraphics[width=\linewidth]{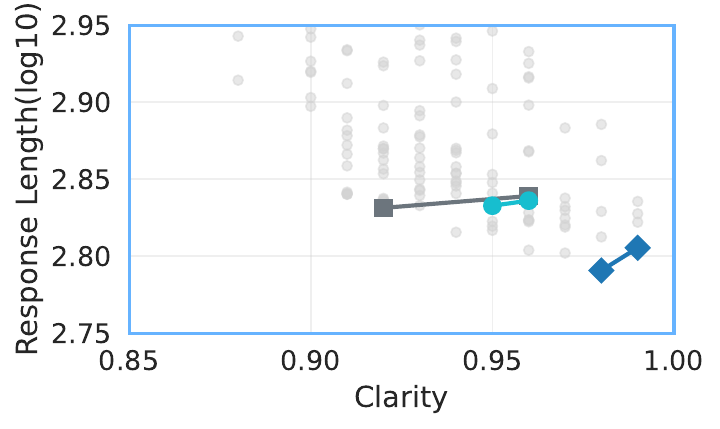}
    \end{subfigure}
    \caption{Pareto fronts of Qwen3-8B trained with GRPO on MATH500 \citep{lightman2023lets} problems.}
  \end{subfigure}

  \begin{subfigure}[t]{\textwidth}
    \centering
    \begin{subfigure}[t]{0.333\textwidth}
      \includegraphics[width=\linewidth]{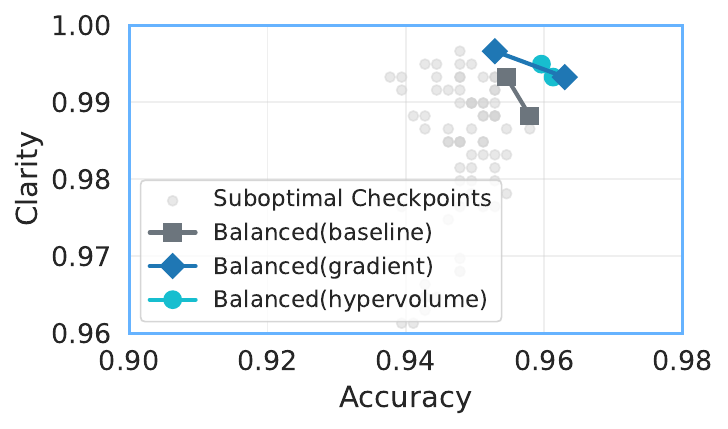}
    \end{subfigure}%
    \begin{subfigure}[t]{0.333\textwidth}
      \includegraphics[width=\linewidth]{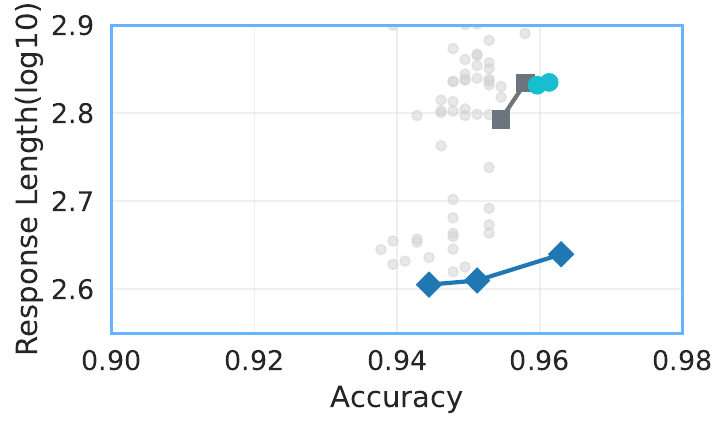}
    \end{subfigure}%
    \begin{subfigure}[t]{0.333\textwidth}
      \includegraphics[width=\linewidth]{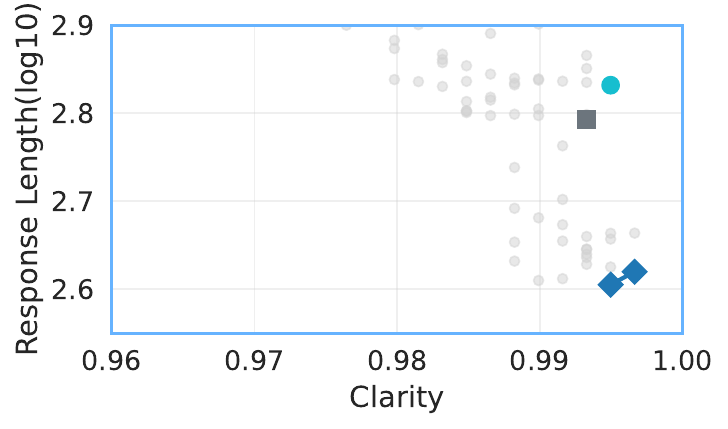}
    \end{subfigure}
    \caption{Pareto fronts of Qwen3-8B trained with GRPO on MATH \citep{mathlighteval} algebra problems.}
  \end{subfigure}

  \vspace{-10pt}
  \caption{Pareto fronts under different datasets.}
  \label{fig: extensive experiment dataset}
\end{figure*}

\begin{figure}[ht]
  \centering
  \includegraphics[width=\linewidth]{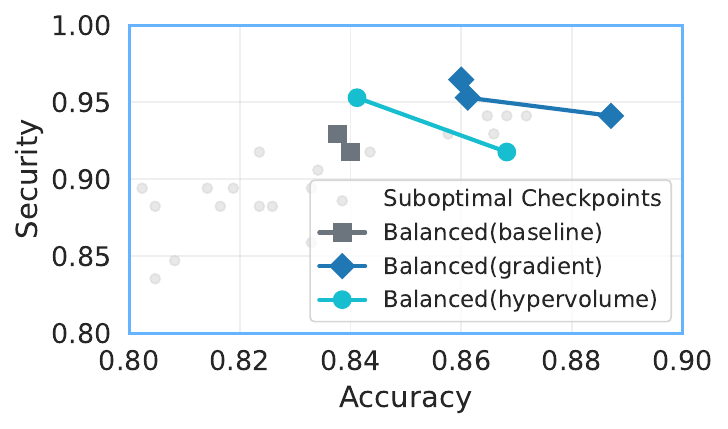}
  \captionsetup{aboveskip=2pt, belowskip=0pt}
  \caption{Pareto fronts of Qwen3-8B trained with SafeSQL \citep{yao2025traininglanguagemodelsgenerate} problems.}
  \label{fig: safesql}
\end{figure}

\begin{figure*}[ht]
  \setlength{\belowcaptionskip}{-3pt}
  \centering

  \captionsetup[subfigure]{skip=1pt, belowskip=8pt}

  \begin{subfigure}[t]{\textwidth}
    \centering
    \begin{subfigure}[t]{0.333\textwidth}
      \includegraphics[width=\linewidth]{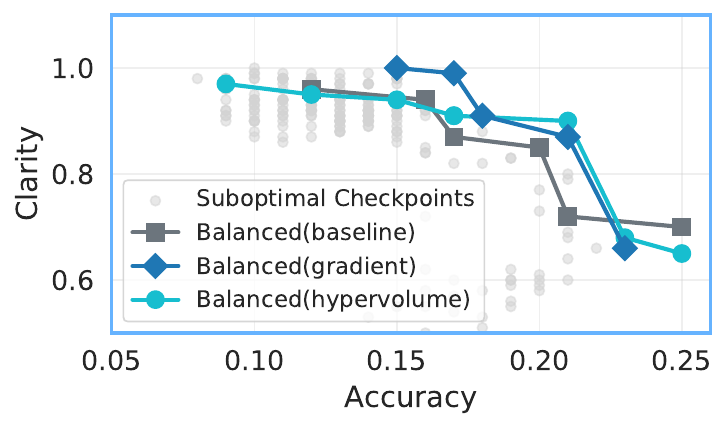}
    \end{subfigure}%
    \begin{subfigure}[t]{0.333\textwidth}
      \includegraphics[width=\linewidth]{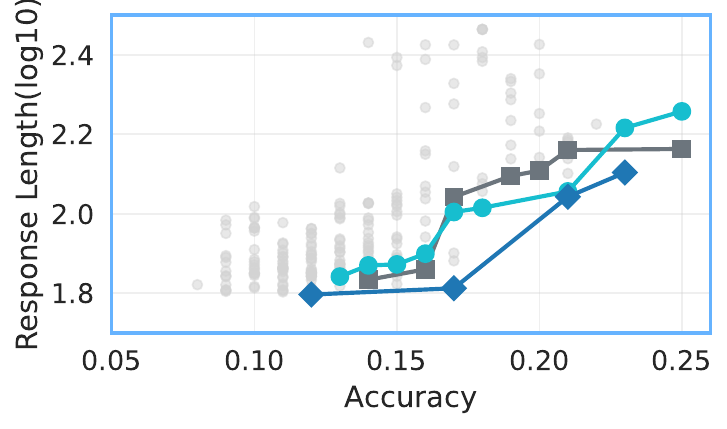}
    \end{subfigure}%
    \begin{subfigure}[t]{0.333\textwidth}
      \includegraphics[width=\linewidth]{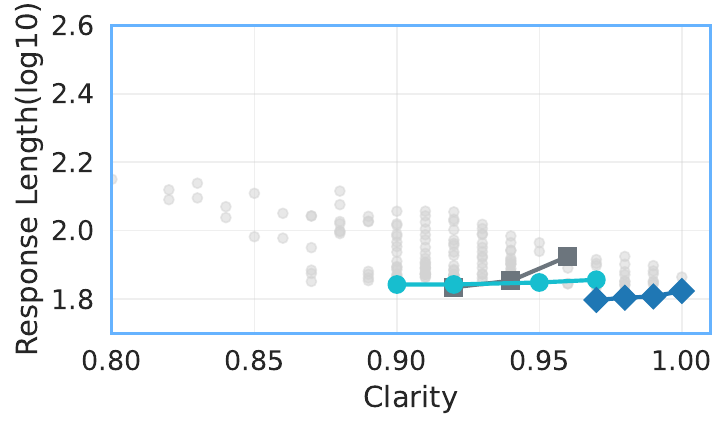}
    \end{subfigure}
    \caption{Pareto fronts of Deepseek-7B \citep{deepseekai2024deepseekllmscalingopensource} trained with GRPO on Math500.}
  \end{subfigure}

  \begin{subfigure}[t]{\textwidth}
    \centering
    \begin{subfigure}[t]{0.333\textwidth}
      \includegraphics[width=\linewidth]{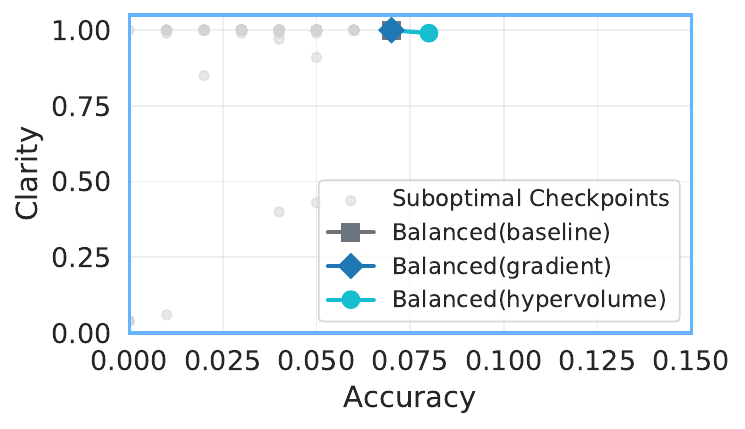}
    \end{subfigure}%
    \begin{subfigure}[t]{0.333\textwidth}
      \includegraphics[width=\linewidth]{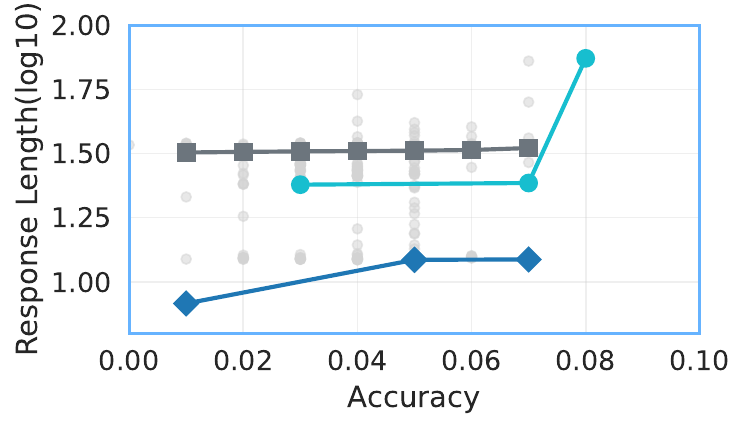}
    \end{subfigure}%
    \begin{subfigure}[t]{0.333\textwidth}
      \includegraphics[width=\linewidth]{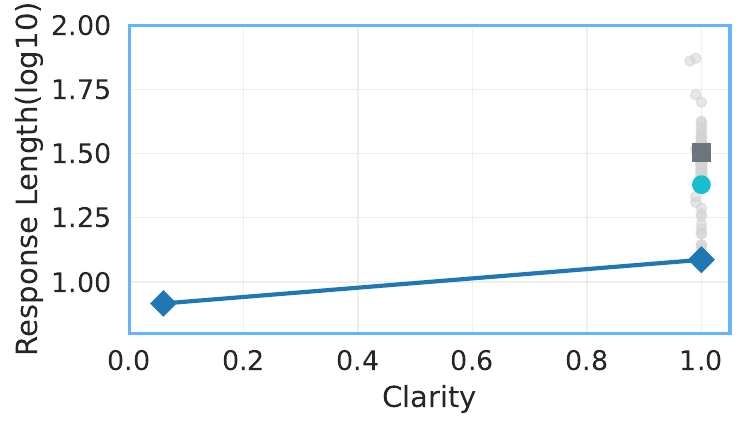}
    \end{subfigure}
    \caption{Pareto fronts of Mistral-7B \citep{jiang2023mistral7b} trained with GRPO on Math500.}
    \label{fig: mistral-7b pareto fronts}
  \end{subfigure}

  \begin{subfigure}[t]{\textwidth}
    \centering
    \begin{subfigure}[t]{0.333\textwidth}
      \includegraphics[width=\linewidth]{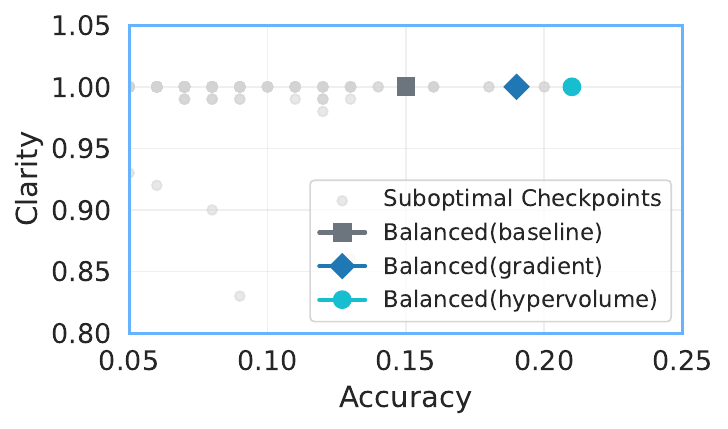}
    \end{subfigure}%
    \begin{subfigure}[t]{0.333\textwidth}
      \includegraphics[width=\linewidth]{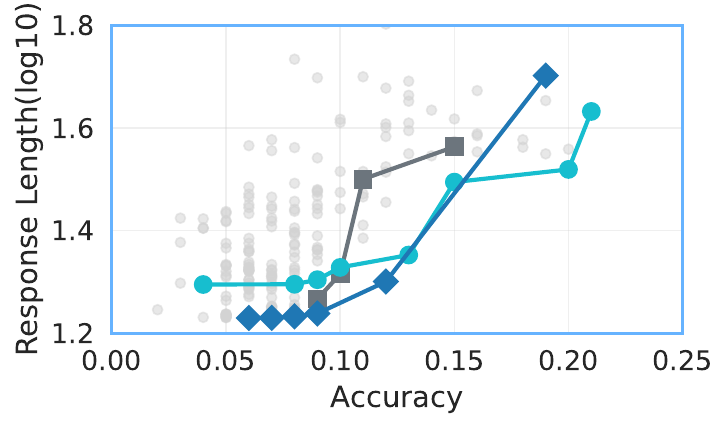}
    \end{subfigure}%
    \begin{subfigure}[t]{0.333\textwidth}
      \includegraphics[width=\linewidth]{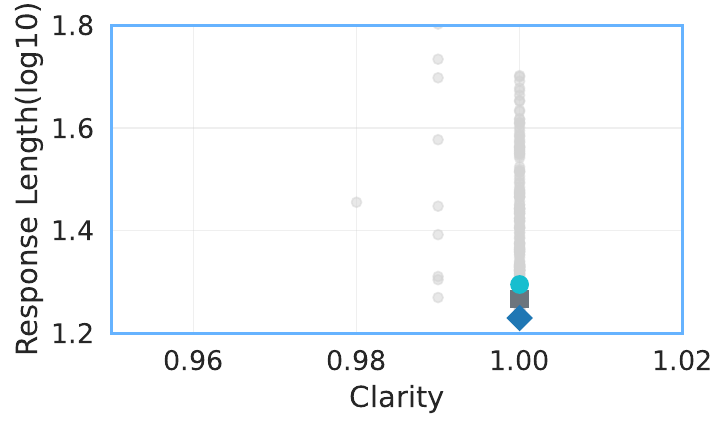}
    \end{subfigure}
    \caption{Pareto fronts of Llama3-8B \citep{grattafiori2024llama3herdmodels} trained with GRPO on Math500.}
    \label{fig: llama3-8b pareto fronts}
  \end{subfigure}

  \vspace{-10pt}
  \caption{Pareto fronts under different models.}
  \label{fig: extensive experiment models}
\end{figure*}

We compare our gradient-based weight optimization approach against baselines trained with different weight configurations in \autoref{table: optimization}. Our method consistently outperforms all baselines in multi-objective alignment across different online RL algorithms. The corresponding Pareto fronts visualized in \autoref{fig: optimization visualization} (Appendix~\ref{appendix: visualization}) further support these findings, demonstrating that our method generates superior Pareto fronts that dominate all baseline approaches under both GRPO and REINFORCE training.

\paragraph{Learning differs for objectives.} We analyze the evolution of objective weights during training to better understand the relative importance of each objective. As shown in \autoref{fig: weight evolution}, the weight for conciseness rapidly converges to approximately 0.2, with the lost weight mostly shifted to accuracy. Consequently, the accuracy weight exhibits continuous growth, which aligns with our intuition that accuracy is a more challenging objective requiring continual learning. In contrast, conciseness improvements can be achieved rapidly (see \autoref{fig: motivation} in Appendix~\ref{appendix: motivation justification}) and thus require less learning effort.
The higher weights for accuracy and clarity compared to conciseness also support our preliminary findings that accuracy and clarity objectives are highly intertwined in the optimization process, effectively playing similar roles in model updates as evidenced by their consistently low KL-divergence trajectories shown in \autoref{fig: kl divergence}, and driving LLM learning towards the same direction. Conversely, conciseness works orthogonally to these objectives, applying a different influence on model training. We provide further analysis in Appendix~\ref{appendix: motivation justification}.

\section{Generalizability and Convergence Rate Tests}
\label{sec: extensive experiment}
\begin{table}[ht]
\centering
\small
\resizebox{0.98\linewidth}{!}{%
\begin{tabular}{@{}lcc@{\hskip 6pt}cc@{}}
\toprule
\textbf{Method $\rightarrow$}
  & \multicolumn{2}{c}{\textbf{Hypervolume-guided}}
  & \multicolumn{2}{c}{\textbf{Gradient-based}} \\ 
\cmidrule(lr){2-3}\cmidrule(lr){4-5}
\textbf{Online RL $\downarrow$}
  & \multicolumn{1}{c}{Baseline}
  & \multicolumn{1}{c}{Ours ($\Delta$)}
  & \multicolumn{1}{c}{Baseline}
  & \multicolumn{1}{c}{Ours ($\Delta$)} \\
\midrule
GRPO       & 66.5 & \textbf{64.0 (-2.5}) & 71.5 & \textbf{62.6 (-8.9}) \\
REINFORCE  & 63.8 & 65.5 (+1.7)          & 68.5 & \textbf{67.2 (-1.3}) \\
RLOO       & 61.1 & 62.5 (+1.4)          & 68.7 & \textbf{60.5 (-8.2}) \\
\bottomrule
\end{tabular}
}
\caption{Average number of training steps required to reach Pareto fronts for Hypervolume-Guided and Optimization methods versus their respective baselines.}
\label{table: pareto front steps}
\end{table}
To demonstrate the training efficiency of dynamic reward weighting, we compute the average number of training steps required to achieve the current Pareto fronts. As illustrated in \autoref{table: pareto front steps}, while the hypervolume-guided method that proactively pushes Pareto fronts has only marginal efficiency gains compared to baselines, the gradient-based method consistently has a higher convergence rate, reducing the required steps by 6.1 on average across RL algorithms.

We further validate the generalizability of our two methods by extending our experiments to two additional datasets spanning mathematical and coding problems (\autoref{fig: extensive experiment dataset} and \autoref{fig: safesql}) and three model families (\autoref{fig: extensive experiment models}). For fair comparison, we use the same constant learning rate across both methods and the baseline. We adopt identical objective sets for the mathematical datasets. For coding tasks, we use SafeSQL, a benchmark designed to evaluate SQL injection vulnerability. Each task requires generating code that constructs SQL queries which both retrieve correct results from a given database and are robust to injection attacks. Accordingly, SafeSQL involves two objectives: \textit{accuracy} and \textit{security}. We report the SafeSQL results in \autoref{fig: safesql}.
\textbf{Similar to our main experiment results, our two methods achieve superior trade-offs compared to the baseline, with the gradient-based weighting showing the best overall performance.} 

Notably, several panels show the Pareto front collapsing to a single point (e.g., accuracy–clarity plots for Mistral-7B in \autoref{fig: mistral-7b pareto fronts} and Llama3-8B in \autoref{fig: llama3-8b pareto fronts}). This reflects a ceiling effect where clarity quickly saturates near its maximum and exhibits little variation across runs. Consequently, Pareto-optimal checkpoints become indistinguishable along the clarity dimension and overlap in two-dimensional projections. The meaningful trade-offs are instead captured by other objective pairs such at accuracy versus response length.

\section{Conclusion and Future Work}
We found that different objectives in multi-objective LLM alignment require varying learning efforts and therefore proposed dynamic reward weighting. We introduced two dynamic reward weighting approaches of increasing sophistication: hypervolume-guided weight adaptation and gradient-based weight optimization. Our experiments across multiple online RL algorithms, datasets, and model families demonstrated that both methods consistently outperform fixed-weight linear scalarization baselines, achieving superior Pareto optimal solutions while improving training efficiency.

\paragraph{Future Work.} While our dynamic reward weighting proves effective across various models, its success depends critically on a model's inherent capacity to simultaneously improve multiple objectives. For instance, when we evaluated Ministral-8B-Instruct \citep{mistralai2024ministral8b_instruct_2410} and Llama-3.1-8B-Instruct \citep{grattafiori2024llama3herdmodels} under identical multi-objective settings, we consistently observe that reducing response length leads to accuracy degradation, revealing a fundamental conflict. In such cases, our methods yield limited gains. We think this limitation likely stems from either performance saturation due to pretraining or insufficient model capacity to effectively learn certain objectives. Both issues are difficult to resolve through post-training alone. Therefore, these findings suggest two promising research directions: (1) developing pretraining strategies that equip models to learn and balance diverse objectives more effectively, and (2) designing post-training approaches for determining which objectives should be prioritized when inreducible conflicts arise.

\section*{Acknowledgement}
We thank the anonymous reviewers and the action
editor, Jesse Thomason, for their valuable feedback.
This work was partially supported by NSF IIS-2119531, IIS-2137396, IIS-2142827, IIS-2234058, and Coefficient Giving. We also appreciate the support from Amazon, the Foundation Models and Applications Lab of Lucy Institute, and ND-IBM Tech Ethics Lab.

\bibliography{reference}
\bibliographystyle{acl_natbib}

\onecolumn

\appendix
\section{Appendix}
\subsection{Proof for Gradient-Based Weight Optimization}
\label{appendix: proof}
We first write optimizing reward weights $\vw\in \mathbb{R}^K$ along the training of the policy model $\theta$ as a bilevel-optimization problem:
\begin{align*}
\vw \in \argmin_{\vw\in \mathbb{R}^K} \sum_{i\in[K]}J_i(\theta^\ast(\vw)), \;\;s.t. \; \theta^\ast(\vw) \in \argmin_\theta \sum_{i\in[K]}w_iJ_i(\theta). 
\end{align*}
The lower-level objective $\theta(\vw)$ is updated using the weights $\vw$, and the higher-level objective $\vw$ is found based on the loss function of the updated model. With a greedy approximation of $\min_{\vw} \sum_i J_i(\theta^{(t)}(\vw))$ (for simplicity, we denote $\theta^{(t)}(\vw) = \theta^{(t)}$ hereafter), we search for the optimal reward weights $\vw^{(t)}$ at step $t$ to minimize the average loss over $K$ objectives at step $t+1$:
\begin{align}
    \argmin_{\vw \in \mathbb{R}^K} \bar{J}(\theta^{(t+1)}) = \argmin_{\vw \in \mathbb{R}^K} \sum_{i\in[K]}J_i(\theta^{(t+1)}) = \argmin_{\vw \in \mathbb{R}^K} \sum_{i\in[K]}\big[J_i(\theta^{(t+1)}) - J_i(\theta^{(t)})\big]
    \label{eq: loss difference}
\end{align}
From the first-order Taylor approximation, we have
\begin{align*}
J_i(\theta^{(t+1)}) &= J_i(\theta^{(t)}) + \nabla J_i(\theta^{(t)})\cdot(\theta^{(t+1)} - \theta^{(t)}) + \mathcal{O}(\|\theta^{(t+1)} - \theta^{(t)}\|) \\
&= J_i(\theta^{(t)}) + \nabla J_i(\theta^{(t)}) \cdot \Big[-\eta^{(t)}\sum_{i\in [K]} w^{(t)}_i\nabla J_i(\theta^{(t)})\Big] + \mathcal{O}(\|\vw\|),
\end{align*}
where $\mathcal{O}(\|\theta^{(t+1)} - \theta^{(t)}\|) = \mathcal{O}(\|\vw\|)$ as the high-order remainder from Tylor approximation. Therefore, we can write Eq.~\ref{eq: loss difference} as:
\begin{align*}
    \argmin_{\vw \in \mathbb{R}^K} \bar{J}(\theta^{(t+1)}) &= \argmin_{\vw \in \mathbb{R}^K} \sum_{i\in [K]} \nabla J_i(\theta^{(t)}) \cdot \Big[-\eta^{(t)}\sum_{k\in [K]} w^{(t)}_k\nabla J_k(\theta^{(t)})\Big] + \mathcal{O}(\|\vw\|) \\
    &= \argmin_{\vw \in \mathbb{R}^K} -\eta^{(t)}\sum_{i\in [K]} \nabla J_i(\theta^{(t)}) \cdot \Big[\sum_{k\in [K]} w^{(t)}_k\nabla J_k(\theta^{(t)})\Big]  + \mathcal{O}(\|\vw\|) \\
    &= \argmin_{\vw \in \mathbb{R}^K} -\eta^{(t)}\sum_{i\in [K]} w^{(t)}_i \Big[\nabla J_i(\theta^{(t)}) \cdot \sum_{k\in [K]} \nabla J_k(\theta^{(t)})\Big]  + \mathcal{O}(\|\vw\|) \\
    &= \argmin_{\vw \in \mathbb{R}^K} -\eta^{(t)} \sum_{i\in [K]} w^{(t)}_i I_i^{(t)} + \mathcal{O}(\|\vw\|), \quad \text{where}\; I_i^{(t)} \triangleq \langle \nabla J_i(\theta^{(t)}), \sum_{k\in [K]} \nabla J_k(\theta^{(t)}) \rangle \\
    &= \argmin_{\vw \in \mathbb{R}^K} -\eta^{(t)} \langle \vw, \mI^{(t)} \rangle + \mathcal{O}(\|\vw\|), \quad \text{where}\; \mI^{(t)} = [I^{(t)}_1, I^{(t)}_2, \ldots, I^{(t)}_K].
\end{align*}
Follow common pratice in mirror descent \citep{BECK2003167}, we estimate $\mathcal{O}(\|\vw\|)$ by introducing a regularization term with factor $\mu$ via Bregman divergence $D_h(\vw\|\vw^{(t-1)}) = h(\vw) - h(\vw^{(t-1)}) - \langle \nabla h(\vw^{(t-1)}), \vw - \vw^{(t-1)}\rangle$, where we employ the entropy function $h(\vw) = \sum_i w_i\ln w_i$. So the above equation becomes:
\begin{align}
\label{eq: final argmin expression}
\vw^{(t)} &:= \argmin_{\vw \in \mathbb{R}^K} -\eta^{(t)} \langle \vw, \mI^{(t)} \rangle + \mu \Big(h(\vw) - h(\vw^{(t-1)}) - \langle \nabla h(\vw^{(t-1)}), \vw - \vw^{(t-1)}\rangle\Big)\\ \nonumber
&=\argmin_{\vw \in \mathbb{R}^K} -\eta^{(t)} \langle \vw, \mI^{(t)} \rangle + \mu\Big(h(\vw) - \langle \nabla h(\vw^{(t-1)}), \vw \rangle\Big).
\end{align}
We then take the derivative of Eq.~\ref{eq: final argmin expression}:
\begin{align*}
    \nabla\big(-\eta^{(t)} \langle \vw, \mI^{(t)} \rangle + \mu\big(h(\vw) - \langle \nabla h(\vw^{(t-1)}), \vw \rangle\big)\big) &= 0 \\
    -\eta^{(t)}\mI^{(t)} + \mu[\ln w + 1]_{i} - \mu[\ln w^{(t-1)} + 1]_i &= 0 \\
    \Rightarrow \ln \vw^{(t)} = \ln \vw^{(t-1)} + \frac{\eta^{(t)}\mI^{(t)}}{\mu}.
\end{align*}
Therefore, we can get the following update rule for reward weights with normalization:
\begin{align}
\vw^{(t)} &= \frac{\vw^{(t)}}{\sum_{k}w_k^{(t)}}, \quad \text{where}\; \vw^{(t)} = \vw^{(t-1)} \odot \exp(\frac{\eta^{(t)}\mI^{(t)}}{\mu}), I_i^{(t)} = \langle \nabla J_i(\theta^{(t)}), \sum_{k\in [K]} \nabla J_k(\theta^{(t)}) \rangle.
\end{align}

\subsection{Preliminary Findings}

\begin{figure*}[ht]
    \centering
    \includegraphics[width=\linewidth]{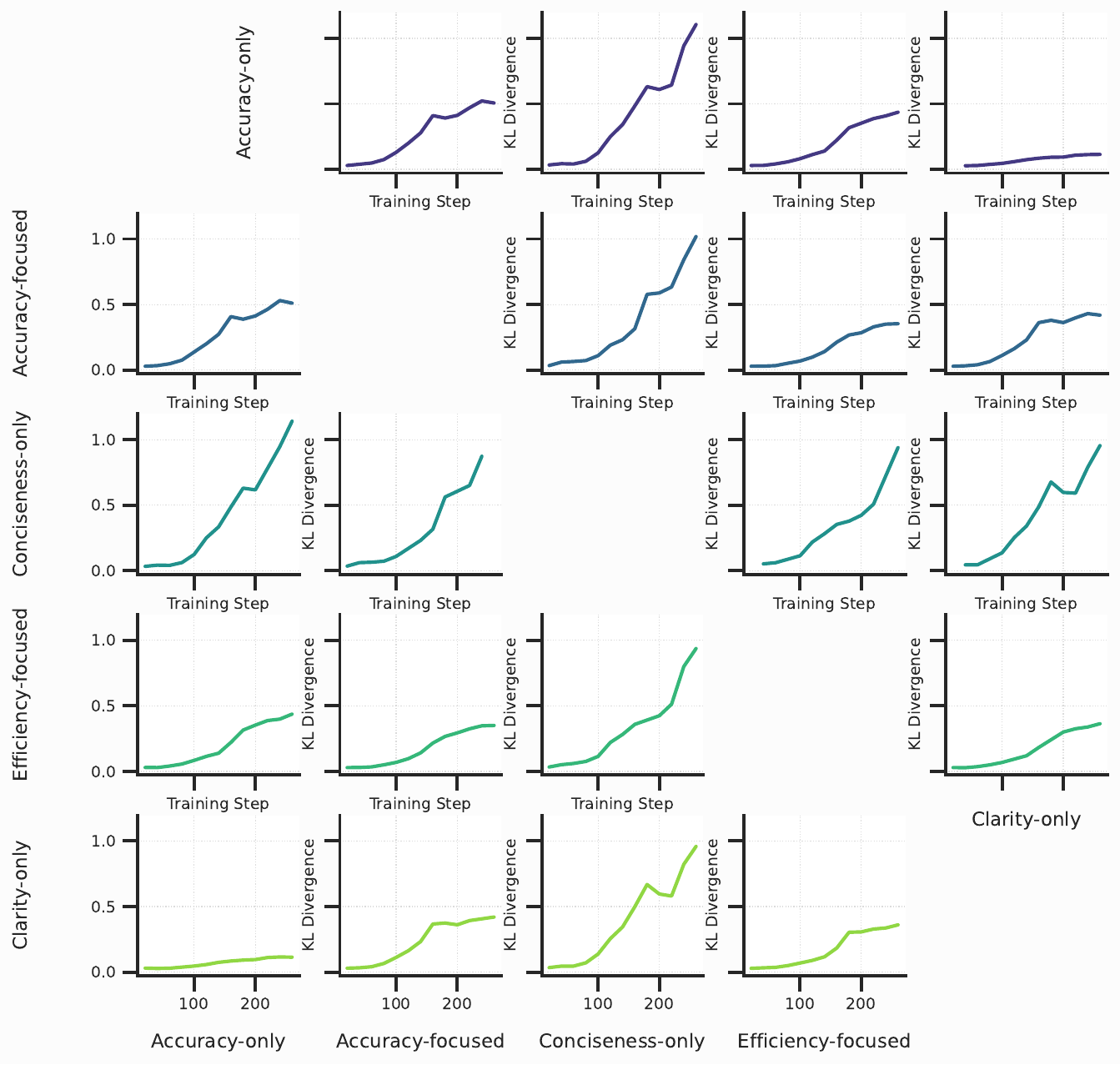}
    \caption{Pairwise KL divergence trajectories between models trained under different weight configurations. Rows represent the model's own weight configurations and columns the configuration it is compared against; within each panel the KL-divergence (y-axis) is plotted over training steps (x-axis).}
    \label{fig: kl divergence}
\end{figure*}
\paragraph{Accuracy and clarity objectives exhibit synergistic effects while conciseness drives learning orthogonally.} To understand the synergistic or antagonistic effect between different objectives in model training, we analyze pairwise KL divergence trajectories between models trained with different weight configurations: accuracy-focused ($0.5,0.25,0.25$), efficiency-focused ($0.25,0.375,0.375$), accuracy-only ($1,0,0$), conciseness-only ($0,1,0$), and clarity-only ($0,0,1$). We conduct this study using the Qwen3-8B model trained on the Math500 dataset via GRPO. As shown in \autoref{fig: kl divergence}, the KL divergence between the accuracy- and clarity-only configurations remains consistently low throughout training, indicating these objectives exhibit strong synergistic effects and induce similar parameter updates. In contrast, models trained with conciseness-only weights diverge significantly from all accuracy- and clarity-only(focused) models, as evidenced by steadily increasing KL divergence over time. This suggests that conciseness applies an orthogonal influence on model optimization, steering parameter updates in a fundamentally different direction compared to accuracy and clarity.
\label{appendix: motivation justification}
\begin{figure}[ht]
  \centering
  \begin{subfigure}[t]{0.333\textwidth}
    \includegraphics[width=\linewidth]{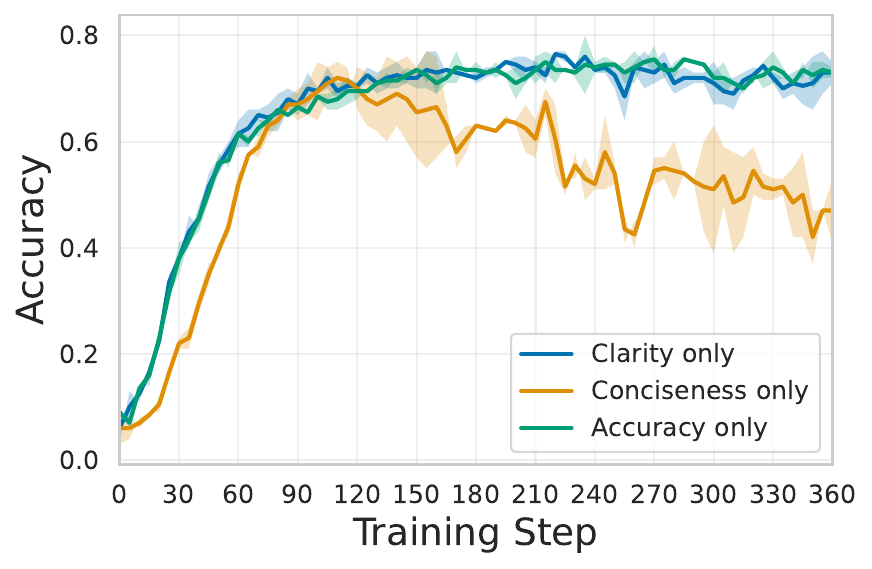}
    \caption{Accuracy}
    \label{subfig: accuracy}
  \end{subfigure}%
  \begin{subfigure}[t]{0.333\textwidth}
    \includegraphics[width=\linewidth]{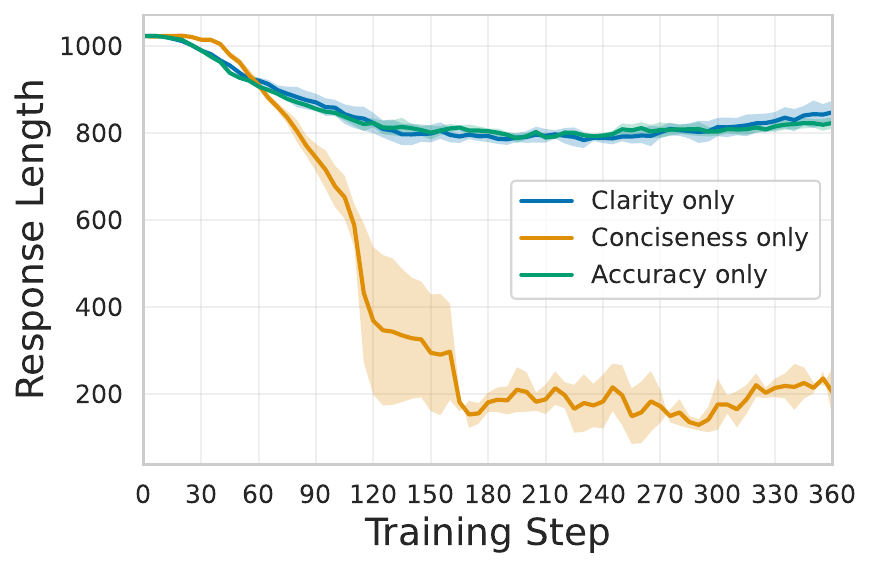}
    \caption{Conciseness}
    \label{subfig: conciseness}
  \end{subfigure}%
  \begin{subfigure}[t]{0.333\textwidth}
    \includegraphics[width=\linewidth]{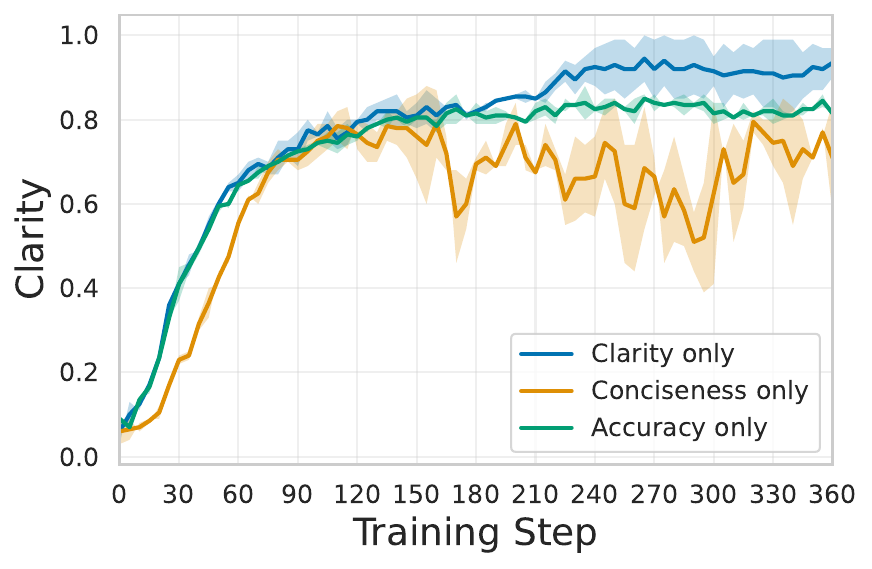}
    \caption{Clarity}
    \label{subfig: clarity}
  \end{subfigure}
    \caption{Validation performance of individual objectives during training. \textbf{Each objective reaches saturation at different training stages.} For example, under the conciseness-only weight configuration, the model achieves optimal response brevity at approximately step 165 (\subref{subfig: conciseness}), whereas the clarity-only configuration reaches the best clarity after step 240 (\subref{subfig: clarity}).}
  \label{fig: motivation}
\end{figure}
\paragraph{Objectives exhibit varying convergence rates.} When training individual objectives, we observe that different objectives reach saturation at different training stages, as illustrated in \autoref{fig: motivation}.

\subsection{hypervolume-Guided Reward Adaptation Details}
\label{appendix: vanilla details}

We compute three types of rewards: (1) accuracy reward through exact matching against the provided ground truth, (2) conciseness reward by comparing current response length to the global average response length from previous rollouts, and (3) clarity reward by checking whether responses contain explicit reasoning steps (e.g., ``first'', ``second'', ``third'') according to our pre-defined rules.

When training the REINFORCE algorithm with a constant learning rate, we observed significant variation in convergence rates across runs using identical hyperparameters, particularly for the conciseness and clarity objectives. To ensure robust and fair comparisons, we repeated runs three times for all baselines and our proposed approach, reporting results from the run that achieved the fastest convergence in each case. This phenomenon was not observed in other algorithms or training settings, and we leave its investigation to future works.

We also observed that zero reward would result in increased response length and degraded performance in both accuracy and formatting, which is likely due to the entropy penalty dominating the gradient updates. To mitigate this issue, we build a minimum threshold of $r_\text{pareto}=0.5$ in Eq.~\ref{eq: r pareto} even when the current checkpoint is Pareto dominated (i.e., $\Delta\text{HV} = 0$).

\subsection{Hyperparameters}
\label{appendix: hyperparameters}
\begin{table}[ht]
    \small
    \centering
    \begin{tabular}{lc}
        \toprule
        Name & Search Bounds \\
        \midrule
        \multirow{2}{*}{$r_\text{pareto}$} 
            & \Big\{$1+\dfrac{1}{1+\exp(-\Delta\text{HV})}$, $\dfrac{2}{1+\exp(-\Delta\text{HV})}$, \\
            & $\mathbf{0.5+1.5\textbf{tanh}(\Delta\text{HV})}$, $0.5+1.5\tanh(3\Delta\text{HV})$\Big\} \\
        \midrule
        Learning rate & $1e^{-6}$ \\
        LR scheduler & constant \\
        Batch size $B$ & 64 \\
        Mini-batch size & 32 \\
        Epoch & 90 \\
        Max response length & 2048 \\
        Rollout size $G$ & 8 \\
        KL coefficient & 0.001 \\ 
        Clip range $\epsilon$ & 0.2 \\
        Dual clip constant $c$ & 3 \\
        Avg. training time & 14 hrs \\
        GPU used & 8 Nvidia H200 (143 GB) \\
        \bottomrule
    \end{tabular}
    \caption{Hyperparameters and other reproducibility information for hypervolume-based weight adaptation.}
    \label{table: hyperparameter vanilla}
\end{table}
\begin{table}[ht]
\small
    \centering
    \begin{tabular}{lc}
        \toprule
        Name & Search Bounds \\
        \midrule
        Learning rate $\eta$ &  \{$1e^{-4}, 1e^{-5}, \mathbf{1e^{-6}}$\} \\
        LR scheduler & \{constant, cosine-annealing, exponential decay, \textbf{polynomial}\}\\
        Poly scheduler power & \{1.01, \textbf{1.03}\} \\
        Regularization factor $\mu$ & \{$1e^{-4}, \mathbf{1e^{-5}}, 1e^{-6}$\} \\
        Clip range $\epsilon$ & \{0.2, \textbf{100}\} \\
        Dual clip constant $c$ & \{3, \textbf{100}\} \\
        \midrule 
        Batch size $B$ & 64 \\
        Mini-batch size & 32 \\
        Epoch & 90 \\
        Max response length & 2048 \\
        Rollout size $G$ & 8 \\
        Max gradient norm & 1.0 \\
        KL coefficient & 0.001 \\ 
        Target layer IDs & $[10, 11, 12, \cdots, 25]$ \\
        Avg. training time & 24 hrs \\
        GPU used & 8 Nvidia H200 (143 GB) \\
        \bottomrule
    \end{tabular}
    \caption{Hyperparameters and other reproducibility information for gradient-based weight optimization.}
    \label{table: hyperparameter optimization}
\end{table}
We use verl \citep{Sheng_2025} for RL training. We provide the hyperparameters used for GRPO training on Math500 in \autoref{table: hyperparameter vanilla} and \autoref{table: hyperparameter optimization}.

\subsection{Pareto Fronts Visualization}
\label{appendix: visualization}

\begin{figure}[ht]
  \centering

  \begin{subfigure}[t]{\linewidth}
    \centering
    \begin{minipage}[t]{0.333\linewidth}
      \centering
      \includegraphics[width=\linewidth]{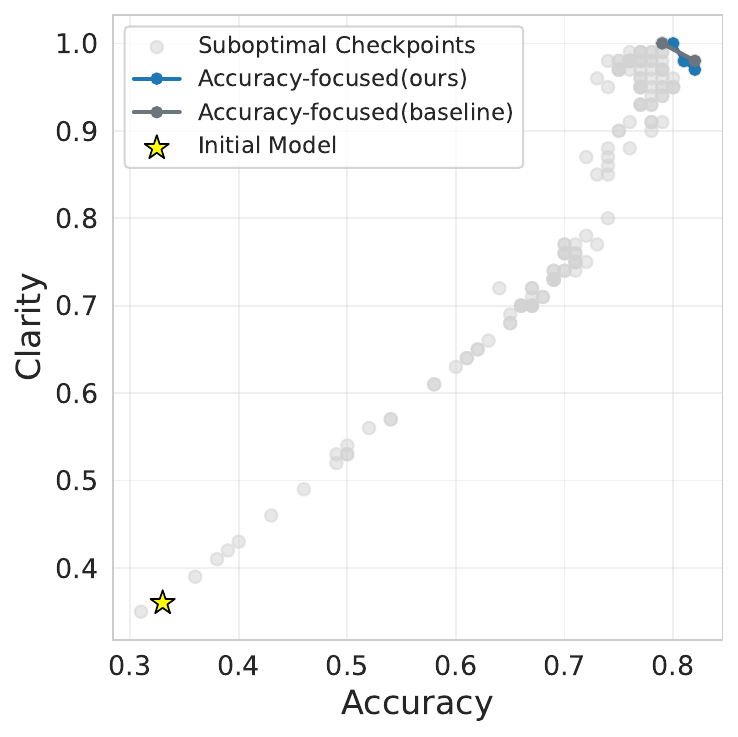}
    \end{minipage}%
    \begin{minipage}[t]{0.333\linewidth}
      \centering
      \includegraphics[width=\linewidth]{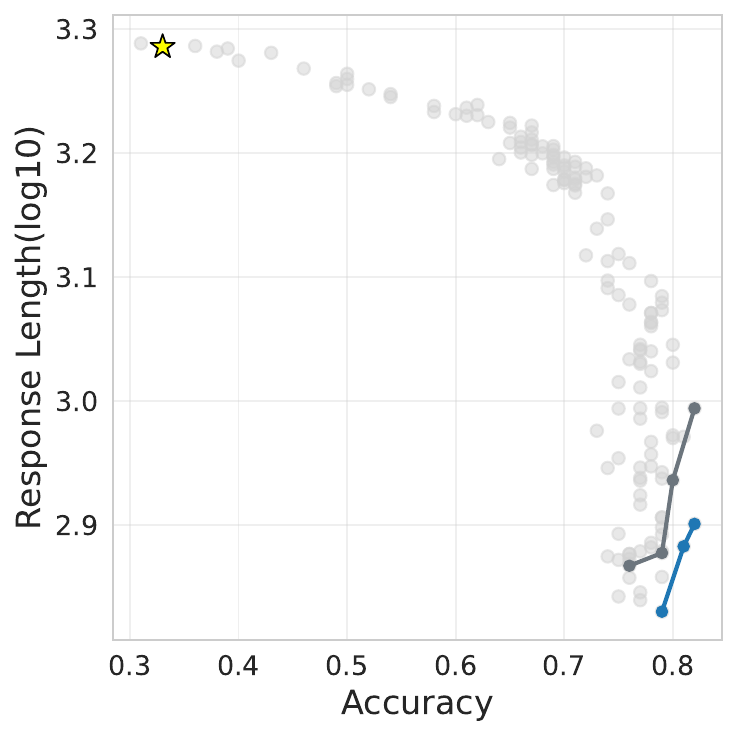}
    \end{minipage}%
    \begin{minipage}[t]{0.333\linewidth}
      \centering
      \includegraphics[width=\linewidth]{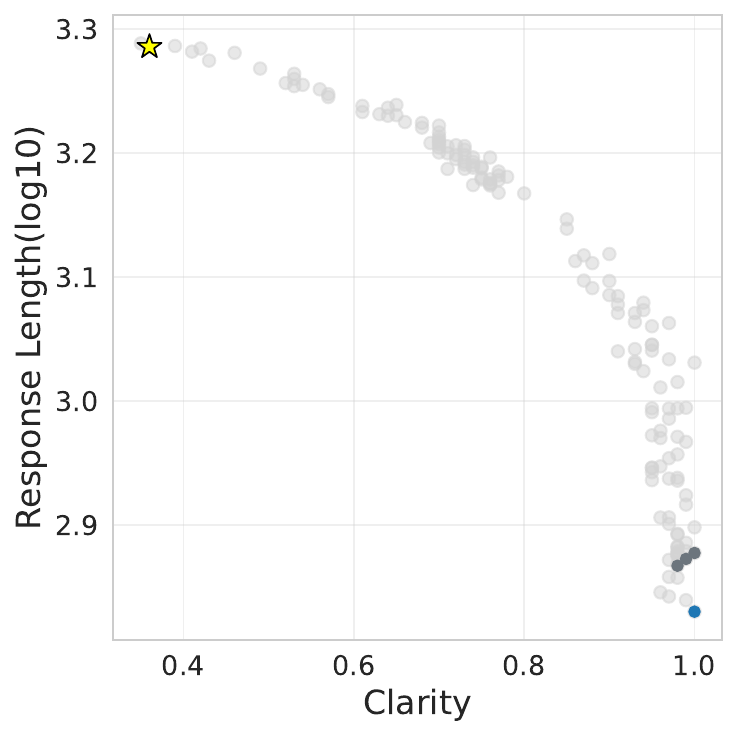}
    \end{minipage}
    \caption{Accuracy-focused.}
    \label{fig:reinforce-row-acc}
  \end{subfigure}

  \vspace{0.6em}

  \begin{subfigure}[t]{\linewidth}
    \centering
    \begin{minipage}[t]{0.333\linewidth}
      \centering
      \includegraphics[width=\linewidth]{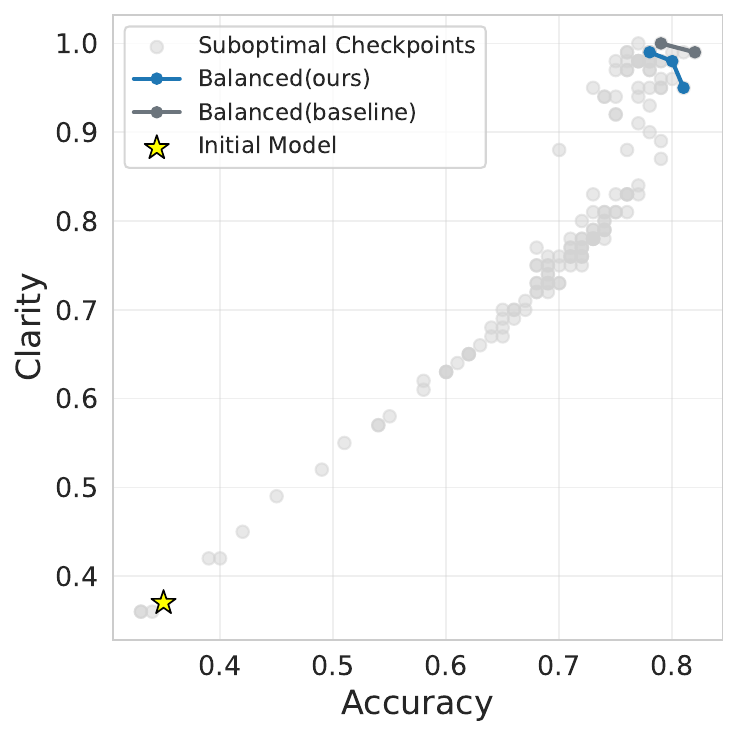}
    \end{minipage}%
    \begin{minipage}[t]{0.333\linewidth}
      \centering
      \includegraphics[width=\linewidth]{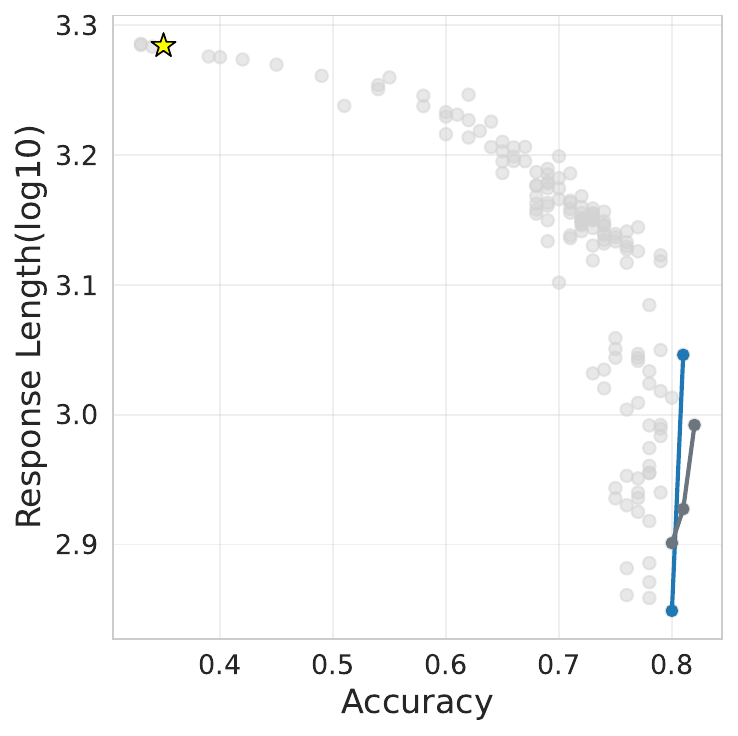}
    \end{minipage}%
    \begin{minipage}[t]{0.333\linewidth}
      \centering
      \includegraphics[width=\linewidth]{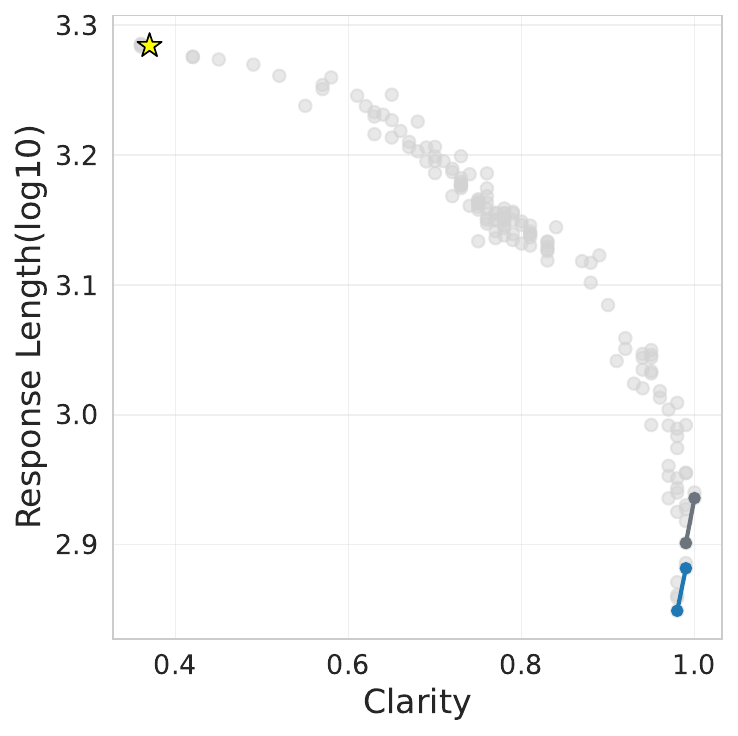}
    \end{minipage}
    \caption{Balanced.}
    \label{fig:reinforce-row-balanced}
  \end{subfigure}

  \vspace{0.6em}

  \begin{subfigure}[t]{\linewidth}
    \centering
    \begin{minipage}[t]{0.333\linewidth}
      \centering
      \includegraphics[width=\linewidth]{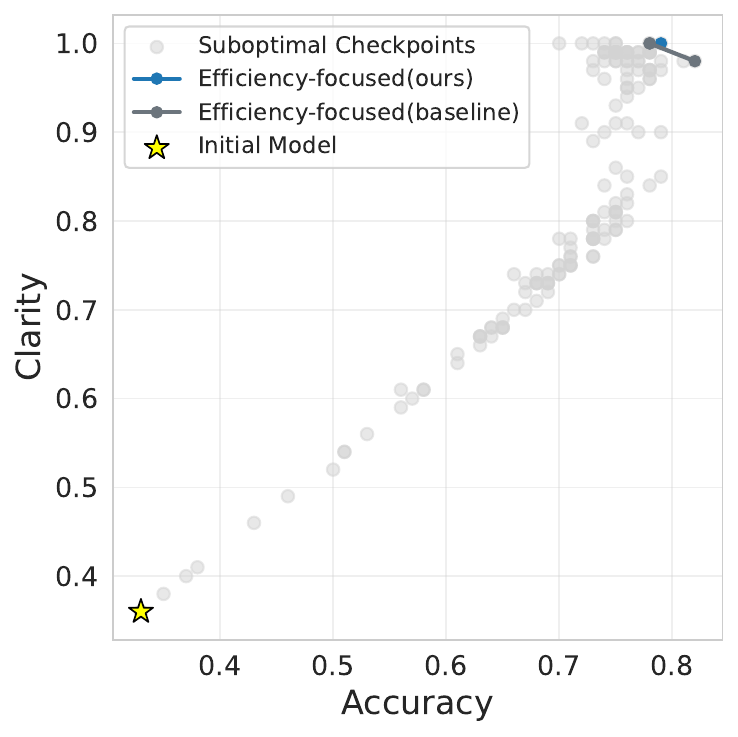}
    \end{minipage}%
    \begin{minipage}[t]{0.333\linewidth}
      \centering
      \includegraphics[width=\linewidth]{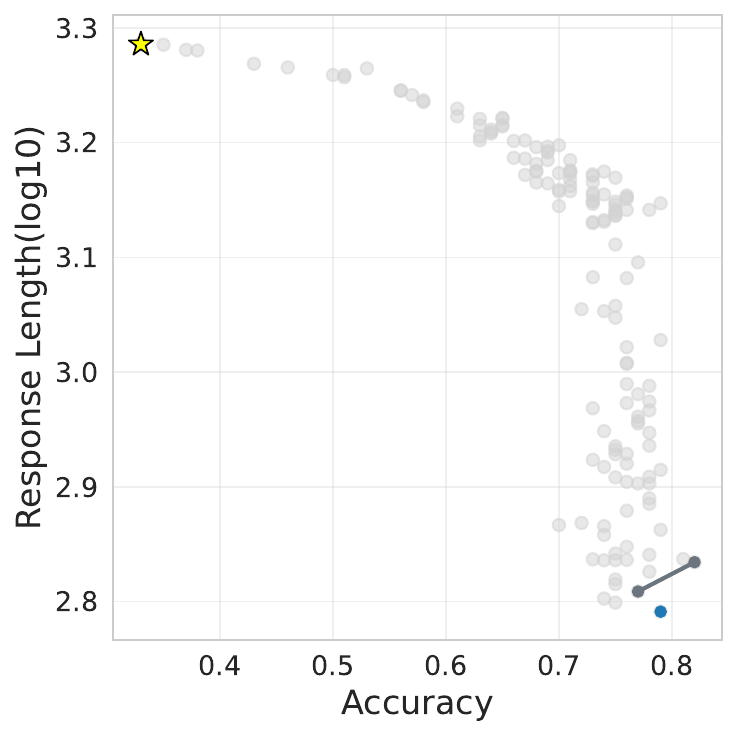}
    \end{minipage}%
    \begin{minipage}[t]{0.333\linewidth}
      \centering
      \includegraphics[width=\linewidth]{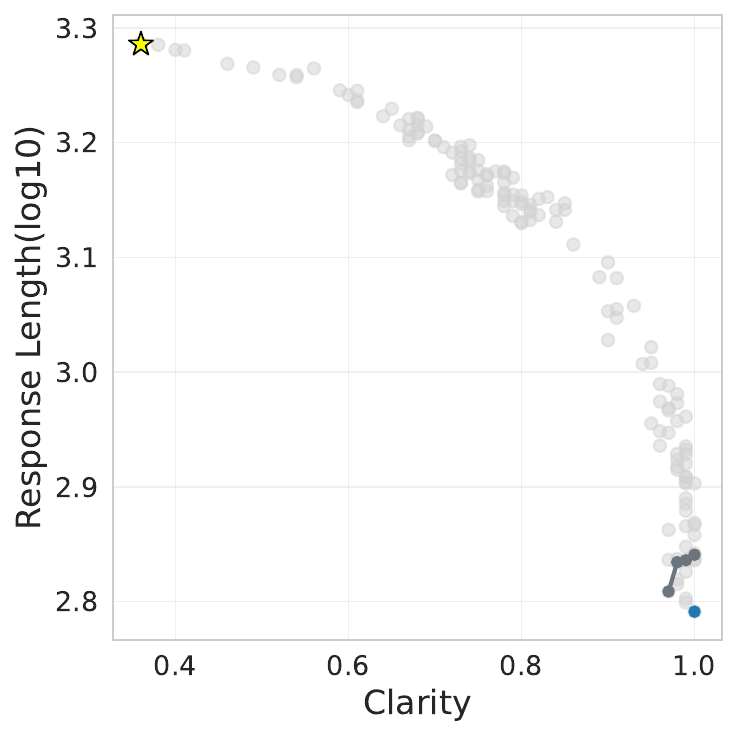}
    \end{minipage}
    \caption{Efficiency-focused.}
    \label{fig:reinforce-row-eff}
  \end{subfigure}

  \caption{Pareto fronts obtained from our hypervolume-guided weight adaptation against baselines under REINFORCE training.}
  \label{fig: reinforce visualization}
\end{figure}

\begin{figure}[ht]
  \centering

  \begin{subfigure}[t]{\linewidth}
    \centering
    \begin{minipage}[t]{0.333\linewidth}
      \centering
      \includegraphics[width=\linewidth]{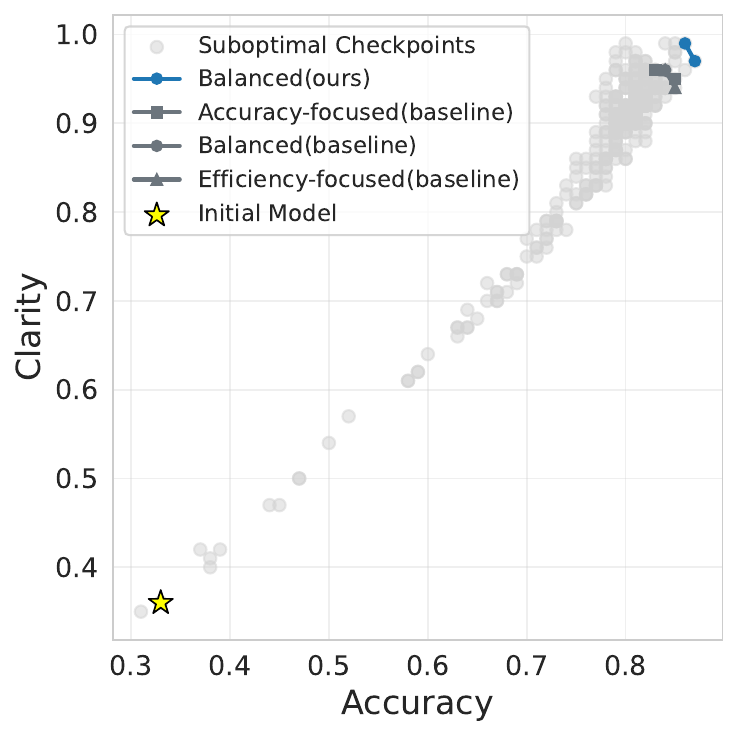}
    \end{minipage}%
    \begin{minipage}[t]{0.333\linewidth}
      \centering
      \includegraphics[width=\linewidth]{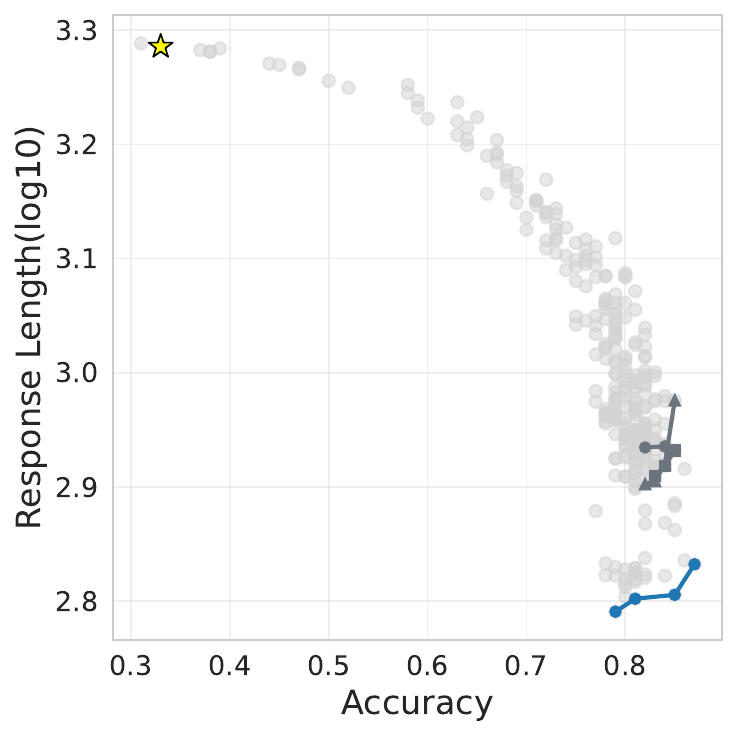}
    \end{minipage}%
    \begin{minipage}[t]{0.333\linewidth}
      \centering
      \includegraphics[width=\linewidth]{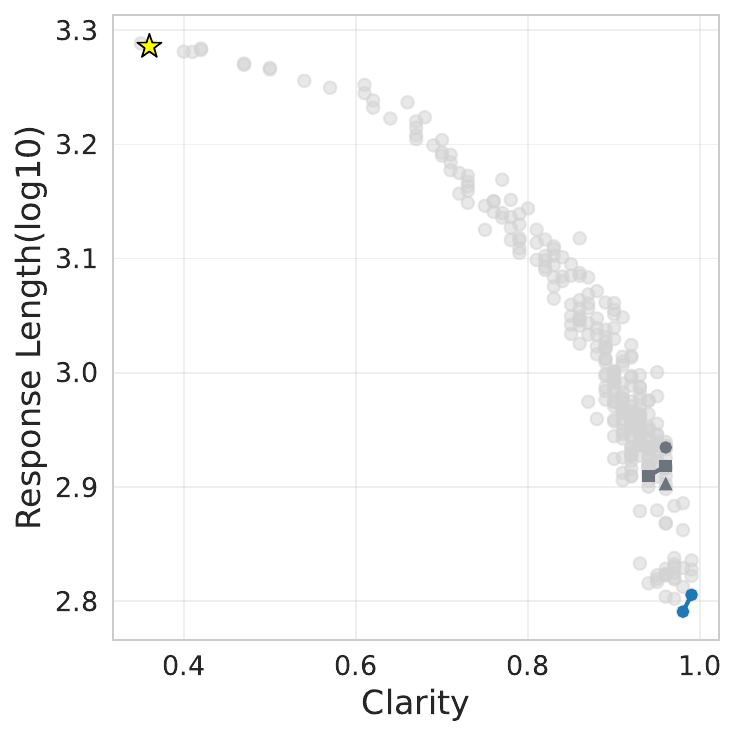}
    \end{minipage}
    \caption{GRPO.}
    \label{fig:opt-row-grpo}
  \end{subfigure}

  \vspace{0.6em}

  \begin{subfigure}[t]{\linewidth}
    \centering
    \begin{minipage}[t]{0.333\linewidth}
      \centering
      \includegraphics[width=\linewidth]{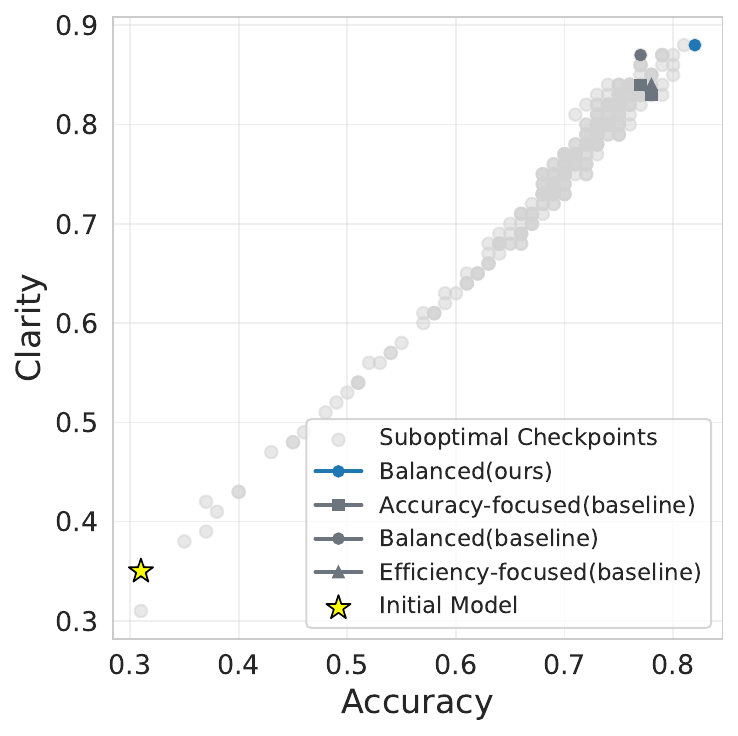}
    \end{minipage}%
    \begin{minipage}[t]{0.333\linewidth}
      \centering
      \includegraphics[width=\linewidth]{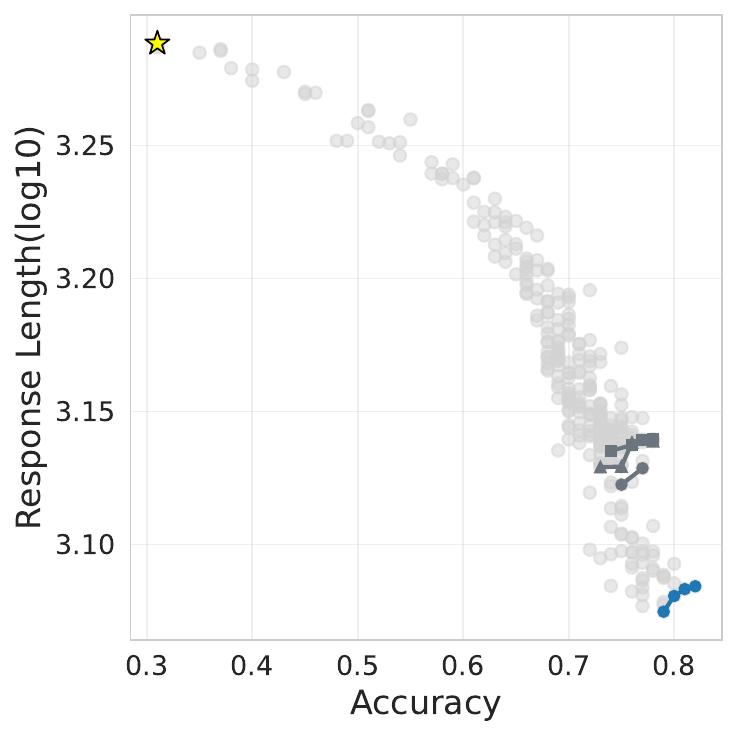}
    \end{minipage}%
    \begin{minipage}[t]{0.333\linewidth}
      \centering
      \includegraphics[width=\linewidth]{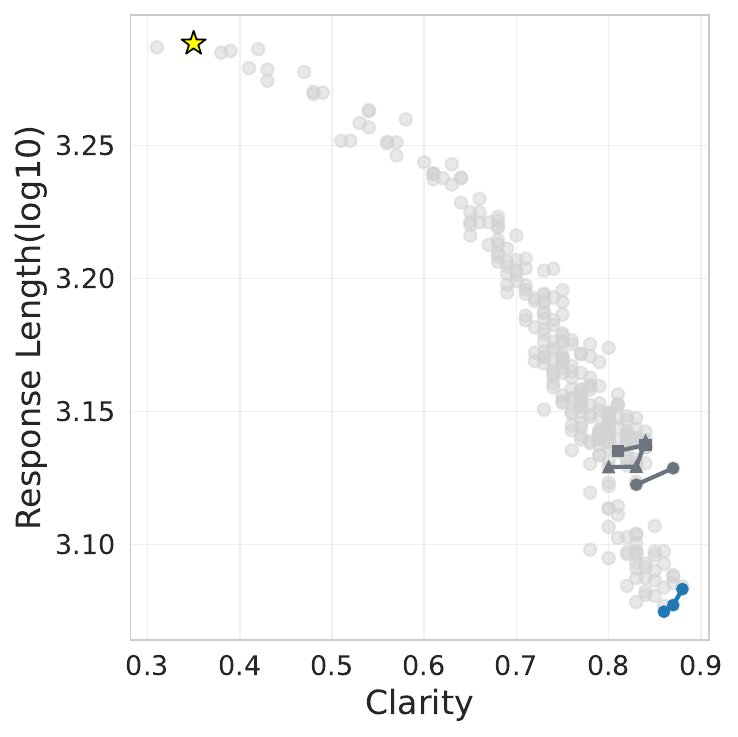}
    \end{minipage}
    \caption{REINFORCE.}
    \label{fig:opt-row-reinforce}
  \end{subfigure}

  \vspace{0.6em}

  \begin{subfigure}[t]{\linewidth}
    \centering
    \begin{minipage}[t]{0.333\linewidth}
      \centering
      \includegraphics[width=\linewidth]{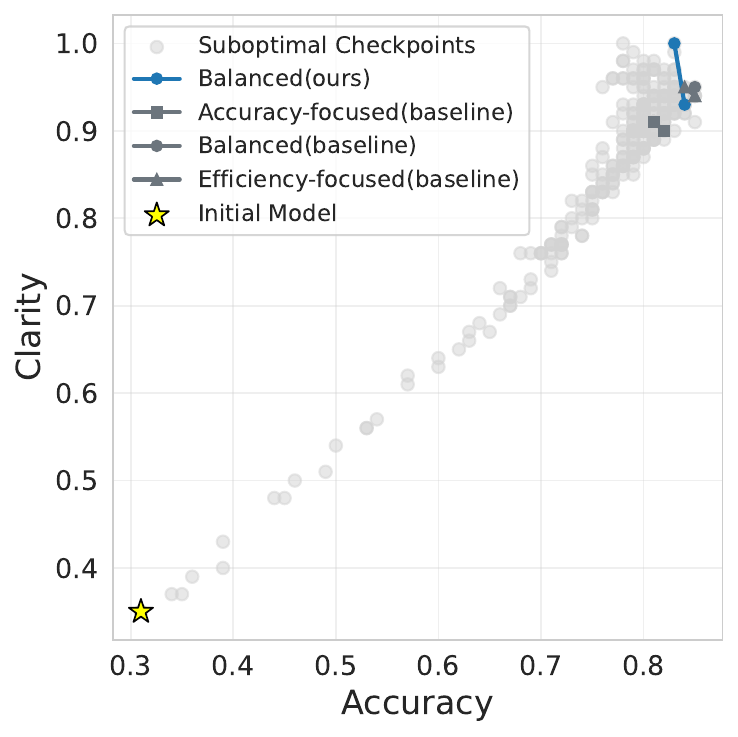}
    \end{minipage}%
    \begin{minipage}[t]{0.333\linewidth}
      \centering
      \includegraphics[width=\linewidth]{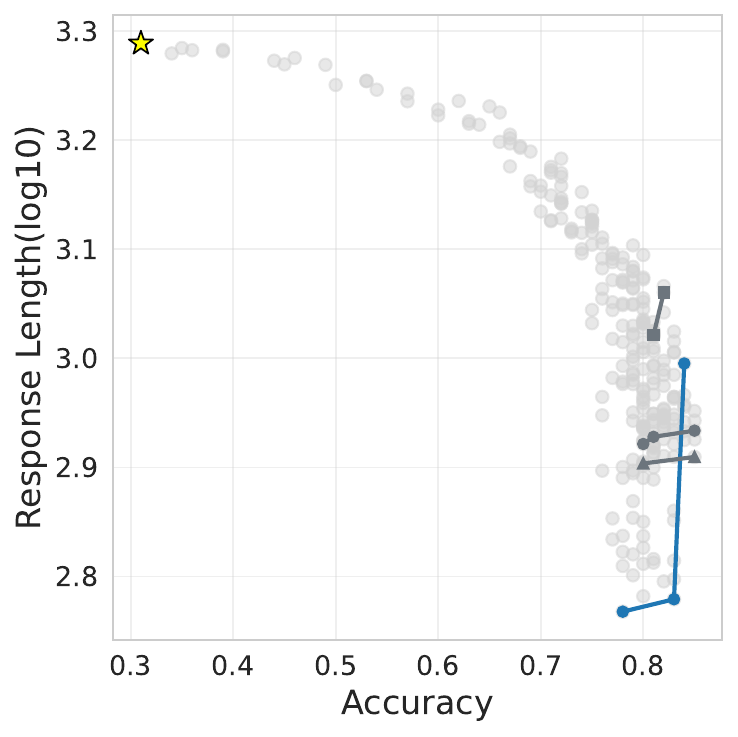}
    \end{minipage}%
    \begin{minipage}[t]{0.333\linewidth}
      \centering
      \includegraphics[width=\linewidth]{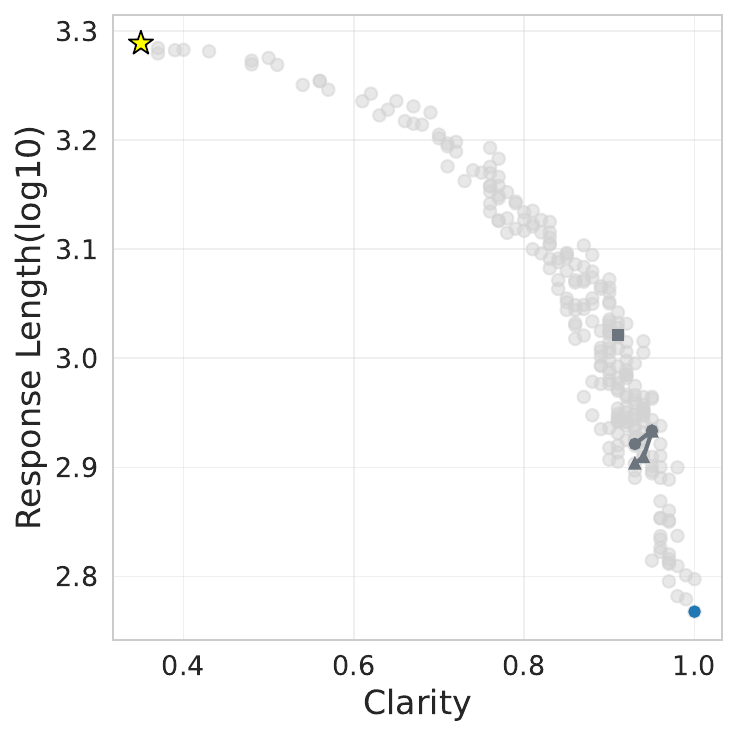}
    \end{minipage}
    \caption{RLOO.}
    \label{fig:opt-row-rloo}
  \end{subfigure}

  \caption{Pareto fronts obtained from our gradient-based weight optimization against baselines under different RL training algorithms.}
  \label{fig: optimization visualization}
\end{figure}

\end{document}

%% file: math_commands.tex
\usepackage{amsmath,amsfonts,bm}

\def\eqref#1{equation~\ref{#1}}

\def\1{\bm{1}}

\def\va{{\bm{a}}}

\def\vr{{\bm{r}}}

\def\vw{{\bm{w}}}

\def\mA{{\bm{A}}}
\def\mB{{\bm{B}}}

\def\mI{{\bm{I}}}

\DeclareMathAlphabet{\mathsfit}{\encodingdefault}{\sfdefault}{m}{sl}
\SetMathAlphabet{\mathsfit}{bold}{\encodingdefault}{\sfdefault}{bx}{n}

\def\sN{{\mathbb{N}}}

\newcommand{\R}{\mathbb{R}}

\DeclareMathOperator*{\argmin}{arg\,min}